\documentclass[journal]{IEEEtran} 

\IEEEoverridecommandlockouts 

\usepackage{graphicx}
\usepackage{tikz}
\usepackage{makecell}
\usepackage{multirow}
\usepackage{amsmath}
\usepackage{amssymb}
\usepackage{amsthm}
\usepackage[ruled,vlined,linesnumbered]{algorithm2e}
\usepackage{subcaption}
\usepackage{mwe} 
\usepackage{hyperref} 
\usepackage{caption}

\def\checkmark{\tikz\fill[scale=0.4](0,.35) -- (.25,0) -- (1,.7) -- (.25,.15) -- cycle;}

\DeclareMathOperator*{\argmin}{arg\,min}
\DeclareMathOperator{\atantwo}{atan2}

\newtheorem{theorem}{Theorem}[section]
\newtheorem{lemma}{Lemma}[section]
\newtheorem{definition}{Definition}[section]
\theoremstyle{definition}

\newcommand{\bbbeta}{\boldsymbol{\beta}}

\newcommand{\bpi}{\boldsymbol{\pi}}
\newcommand{\bU}{\boldsymbol{\xi}}
\newcommand{\bnu}{\boldsymbol{\nu}}

\newcommand{\state}{\mathbf{x}}

\newcommand{\goalstate}{\mathbf{x_G}}
\newcommand{\inp}{\mathbf{u}}
\newcommand{\sensor}{\mathbf{y}}
\newcommand{\sensorpredict}{\hat{\mathbf{y}}}
\newcommand{\obs}{\mathbf{z}}
\newcommand{\obst}{O}

\newcommand{\actorparams}{\boldsymbol{\phi}}
\newcommand{\criticparams}{\boldsymbol{\psi}}
\newcommand{\traj}{\boldsymbol{\tau}}

\newcommand{\EV}{\mathbb{E}\left[V(\state_{t+1}) \mid \state_{t}, \inp_{t} \right]}
\newcommand{\sensorhist}{N_y}
\newcommand{\horizon}{N_T}

\newcommand{\sensorfunc}{\mathbf{h}}
\newcommand{\costfunc}{\ell}
\newcommand{\constraintfunc}{c}
\newcommand{\rlcostfunc}{g}
\newcommand{\actor}{\boldsymbol{\pi}^{\actorparams}}
\newcommand{\critic}{Q^{\criticparams}}
\newcommand{\approxvaluefunc}{V^{\actorparams\criticparams}}

\newcommand{\sensorpredictfunc}{\sensorfunc^\eta}
\newcommand{\nominaldynamicsfunc}{\mathbf{f}}

\newcommand{\cmark}{\textcolor{green!60!black}{\ding{51}}}
\newcommand{\xmark}{\textcolor{red}{\ding{55}}}
\usepackage{pifont}

\renewcommand{\baselinestretch}{0.99}

\newcommand{\blindreview}[1]{}
\newcommand{\original}[1]{{#1}}

\title{\LARGE \bf RL-Guided PAC-NMPC for Probabilistically-Safe Perception-Based Navigation in Unknown Environments}

\original{
\author{
Adam Polevoy$^{1,2}$,
Dillon Capalongo$^2$,
Katherine Tang$^2$,
Mark Gonzales$^2$,
Marin Kobilarov$^2$,
and Joseph Moore$^{1,2}$%
\thanks{$^{1}$Johns Hopkins University Applied Physics Laboratory, Laurel, MD
	20723, USA.}%
\thanks{$^{2}$Department of Mechanical Engineering, Johns Hopkins University, Baltimore, MD 21218, USA.}%
\thanks{ Email: {\tt\small Adam.Polevoy@jhuapl.edu, dcapalo1@jhu.edu, ktang21@jh.edu, mgonza60@jh.edu, marin@jhu.edu, jlmoore@jhu.edu}}%
}
}

\blindreview{
\author{
[Placeholder for Author List]
\thanks{[Placeholder for Author Affiliations]}%
}
}

\begin{document}

\begin{titlepage}
\vspace*{\fill}
{\large
\copyright 2026 IEEE.  Personal use of this material is permitted.  Permission from IEEE must be obtained for all other uses, in any current or future media, including reprinting/republishing this material for advertising or promotional purposes, creating new collective works, for resale or redistribution to servers or lists, or reuse of any copyrighted component of this work in other works.}
\vspace*{\fill}
\end{titlepage}

\bstctlcite{BSTcontrol}

\maketitle


\begin{abstract}

In this paper, we present an approach for combining stochastic nonlinear model predictive control (SNMPC) and reinforcement learning (RL) to enable probabilistically-safe perception-based navigation in unknown environments. Our method first uses RL to train probabilistic actor-critic and sensor prediction models. We then leverage these probabilistic models in a sampling-based SNMPC framework known as Probably Approximately Correct (PAC)-NMPC, which uses hard constraints to enforce finite-time statistical guarantees on the probability of collision and value function improvement. By ensuring that our finite-horizon SNMPC policies decrease the value function in expectation, we can approach the long-horizon performance of the RL approach while satisfying probabilistic safety constraints. Through simulation experiments, we show that our approach can improve the safety of perception-based RL navigation policies and scale to high dimensional systems with large sensor input spaces and complex nonlinear dynamics. We also demonstrate our approach through hardware experiments, showing improved performance for vision-based navigation with an agile fixed-wing aerial vehicle in unknown environments.

\end{abstract}



\section{Introduction}

Autonomous perception-based navigation in unknown environments remains a fundamental research challenge for agile, underactuated robots characterized by complex nonlinear dynamics and sizable state spaces. To navigate at-speed, these robots must predictively reason about their underlying dynamics while balancing computational efficiency, performance, and safety. During real-world deployment, not only must they consider the uncertainty in their own dynamics, but also the uncertainty associated with perception, including the uncertainty introduced by a limited field-of-view. Moreover, for many safety-critical applications, these algorithms must also be able to provide performance guarantees on metrics like collision-avoidance and task completion. 


Nonlinear model predictive control (NMPC) is one approach for achieving perception-based navigation for a large class of robotic systems. Given a model of the system dynamics and an objective function, NMPC generates control inputs by repeatedly solving a finite-horizon optimization problem which can also be subject to nonlinear constraints. Robust and stochastic NMPC (RNMPC, SNMPC) explicitly model uncertainty in the underlying dynamical system, and in some cases, can provide safety guarantees (e.g., \cite{mayne2011tube, kohler2020computationally, ozaki2020tube}). Often, such guarantees rely on restrictive assumptions about the structure of the dynamics and uncertainty. In addition, the computational burden of NMPC scales with the complexity of the system dynamics model. As a result, NMPC is often limited to short finite-horizons, resulting in policies with myopic behavior which can significantly limit long-range performance. In some cases, NMPC has been combined with sensor models for perception-based navigation and control (e.g., \cite{9812099,8593739}).

\begin{figure}[t]
    \centering
    \includegraphics[trim={0 240 600 0},clip,width=1.0\columnwidth]{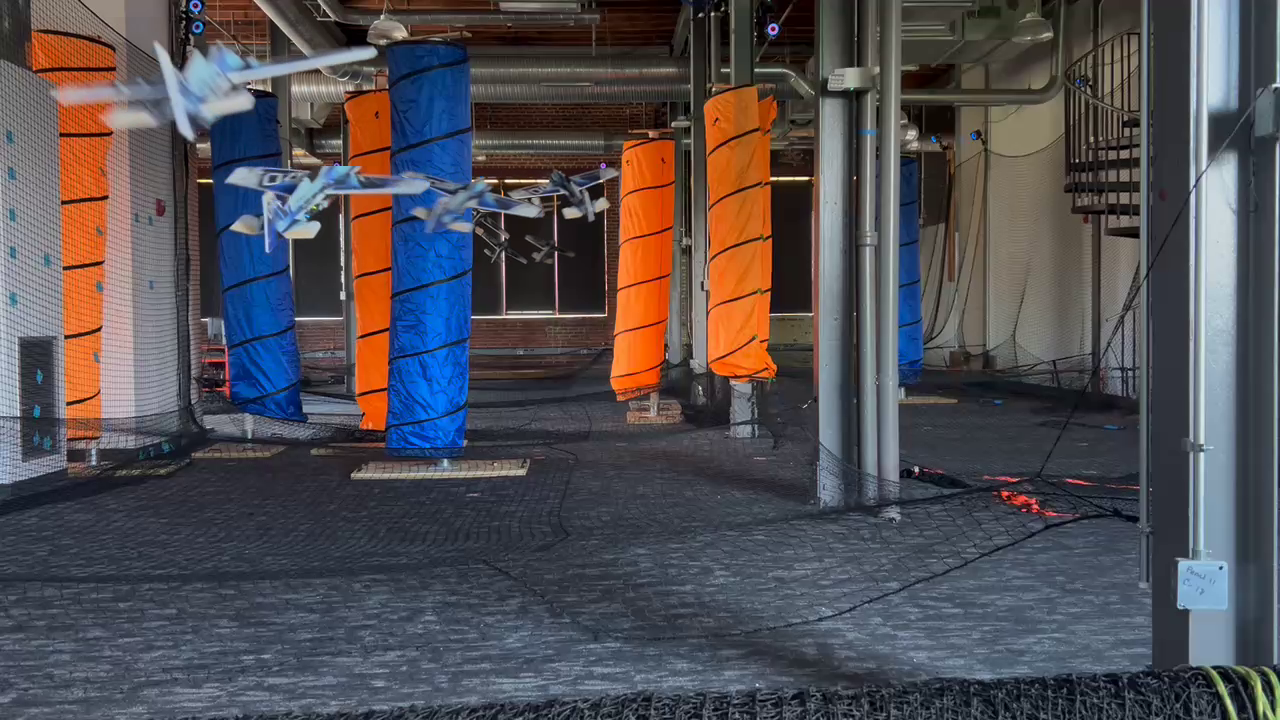}
    \caption{Actor-Critic PAC-NMPC for navigation and obstacle avoidance on a fixed-wing with an on-board depth camera.}
    \label{fig:main-img}
\end{figure}

Reinforcement learning (RL) is another approach for achieving perception-based robot navigation for a general class of robotic systems. Modern, deep RL policies are often trained prior to controller deployment, thus eliminating the need for real-time optimization. Since the policies are often trained to minimize accumulated cost over a discounted infinite horizon, RL approaches typically show improved long-range performance. However, when these policies are only optimized prior to deployment, they can also exhibit degraded performance when operating outside of the training distribution. In general, RL approaches also lack a mechanism to explicitly enforce hard constraints, especially in out-of-distribution environments. Furthermore, because the data requirements of RL often necessitate training in simulation, there is frequently a distribution shift during sim-to-real transfer.

In this article, we present an approach that uses RL-trained networks within a sampling-based SNMPC framework to achieve probabilistically-safe perception-based navigation. In particular, we use a sampling-based SNMPC approach known as Probably Approximately Correct (PAC)-NMPC, which can provide statistical guarantees on cost and constraint satisfaction to ensure safety according to Def. \ref{def:safetyguarantee}. Instead of training the RL networks jointly with the PAC-NMPC policy, we train the RL networks independently and use Monte-Carlo (MC) dropout to approximate the policy and value function distributions. For high-dimensional sensor inputs, we also train a generative sensor network to predict future (unobserved) sensor returns. Given these learned probabilistic models, during deployment, we then warm-start our SNMPC approach using the RL policy distribution and sample terminal SNMPC costs from the value function distribution. 

By combining SNMPC and RL in this way, our approach achieves three distinct advantages. First, it improves the safety of the RL policy according to Def. \ref{def:safetyguarantee} by using PAC-NMPC to provide statistical finite-time run-time guarantees on constraint satisfaction (e.g., local collision avoidance). Second, it also uses PAC-NMPC to provide statistical guarantees on value function improvement over the finite SNMPC horizon and dramatically improves the long-range optimality of the SNMPC policy. Third, by warm-starting the SNMPC policy with the learned actor and using a predictive sensor model, the SNMPC algorithm can be decoupled from the RL policy during training and thus remove the computational burden of an inner NMPC optimization loop. We demonstrate the performance improvements of our approach through simulation and hardware experiments and evaluate our method on the particularly challenging problem of vision-based collision-free navigation with an agile fixed-wing aerial robot.  

This manuscript substantially extends the work in \cite{pacnmpc-lvf}, which introduced an approach for combining PAC-NMPC with a learned stochastic value function. In this paper, we significantly advance this approach to scale to higher state and measurement spaces and three-dimensional environments, and to provide finite-time guarantees in unknown environments. Our key contributions are as follows: 
\begin{enumerate}
\item A constrained PAC-based framework for combining SNMPC with RL to enable both long-horizon planning and enforcement of statistical safety guarantees. 
\item A learned model to predict future sensor measurements, for which closed form dynamics are unavailable in unknown environments
\item An RL-based SNMPC policy warm-starting method for scaling to higher dimensional state measurement spaces.
\item A hardware evaluation of our approach on a fixed-wing aerial vehicle with an on-board depth camera, demonstrating real-time control of robotic systems with local perception and complex nonlinear dynamics. 
\end{enumerate}


\section{Related Work}

A number of approaches have been proposed to improve the safety of RL policies or improve the long-range performance of receding-horizon NMPC. In this section, we provide a brief overview of prior related research on safe RL, learned NMPC warm-starting, NMPC with learned waypoints, and hybrid RL-NMPC methods. We also review prior research on perception-based fixed-wing navigation, since we use this unique planning and control problem to evaluate our approach.
\subsection{Safe RL}
Safe RL methods aim to improve the safety of RL policies by training and/or deploying learned policies that minimize expected future costs while adhering to safety constraints. These problems are often posed as Constrained Markov Decision Processes (CMDP). In \cite{gu-safe-rl-survey}, the authors provide an overview of recent safe RL methods and in \cite{brunke-safe-rl-survey}, the authors provide a more general review of learned safe control under uncertainty. 

Safe policy optimization methods rely on data-driven approaches to learn safe policies. Constrained policy gradient methods \cite{safe-rl-cpo} modify the policy update in an attempt to maintain satisfaction of expected future constraints. Other approaches use Lagrangian relaxation to approximate the problem as unconstrained optimization \cite{safe-rl-lagrange}. However, once outside of the training distribution, these approaches can no longer guarantee constraint satisfaction. 

Control theory based methods introduce explicit models to regulate inputs to ensure safety, thus providing stronger assurances. Many approaches require knowledge of a Lyapunov function in advance to guarantee safety \cite{safe-rl-lyapunov}. Likewise, control barrier functions (CBFs) are often utilized to guarantee safety \cite{safe-rl-cbf}. While the CBFs can be learned from data, they assume that safety can be represented completely as a function of the state, which may not be possible when only partial observations of the environment are available. Safety layer methods take possibly unsafe inputs and project them to inputs that satisfy constraints. OptLayer  follows this methodology by augmenting the policy with a constrained optimization layer \cite{safe-rl-layer}. While this framework is similar to our approach, it only guarantees one-step safety and does not consider dynamics or perception uncertainty. 

Instead of trying to enforce safety constraints, other approaches aim to provide a statistical performance guarantee for the learned policy. PAC-NMPC takes inspiration from PAC Robust Policy Search \cite{props, acprops}, which optimizes a PAC guarantee on the expected cost of the policy. However, these guarantees do not necessarily hold outside of the training distribution. PAC-Bayes control \cite{pacbayes} provides both PAC guarantees on policy performance and generalization guarantees to novel environments by assuming bounded divergence between the training and testing distributions.

\subsection{Learned NMPC Warm-start}

Compared to RL, NMPC can more naturally enforce runtime safety constraints. However, the global optimality of these algorithms is typically limited by computational resources and susceptibility to local minima. To improve this long-range performance, researchers have investigated methods that use machine learning to warm-start the NMPC optimization. Many approaches build data sets of optimal trajectories offline and train models to predict policies or trajectories \cite{stasse2018, hauser2018}. In \cite{pavone2020}, the authors used trajectory data to learn warm-starts for sequential convex programming with obstacle avoidance constraints. Researchers have also learned value functions to select from a history of stored controller data while providing stability guarantees \cite{evenbauer2020}.

Other approaches have focused on initializing not only the control input sequence, but also the associated constraints. In \cite{kvasnica2019}, the authors learn a warm start for the active set, and in \cite{morari2022}, authors scaled learned primal active set initialization to higher dimensional systems. More recently, authors developed a set of constraint-informed merit functions for training, thus making feasibility the target of the warm start \cite{cauligi2025}.

Researchers have also begun to explore using more advanced generative models for learned warm-starting. For instance, in \cite{pavone2024}, researchers used a transformer-based architecture to learn a terminal cost for long horizon guidance. In \cite{reichardt2024}, the authors use Motion Transformer to generate multi-modal warm-starts for escaping local minima in fast-changing traffic scenarios. TransformerMPC \cite{transformer-mpc} improved optimization solve time by training a transformer to select only active constraints in the MPC problem. In \cite{petrik2026}, authors proposed a diffusion-based approach conditioned on object-centric perception data. 

While learned warm-starting has been applied to many diverse and challenging problems such as bipedal locomotion \cite{calinon2020}, self-driving cars \cite{reichardt2024}, and free-flying space robotic platforms \cite{pavone2025}, few of these approaches have considered planning with local perception in unknown environments.

\subsection{NMPC with Learned Waypoints}

In addition to learned warm-starting, researchers have utilized learned waypoints to improve the long-range navigation capabilities of finite-horizon NMPC. In some cases, researchers have used reinforcement learning to generate these waypoints for navigation in unknown environments \cite{sharma2012autonomous}, quadrotor navigation in environments with dead-end corridors \cite{greatwood2019reinforcement}, and navigation in dynamic environments with other agents \cite{brito2021go}. In other cases, supervised learning has been combined with offline perception-based kinodynamic planning to enable drone racing \cite{kaufmann2018deep} and  navigation in real-world cluttered environments \cite{bansal2020combining}. These approaches often assume a separation of the high-level and low-level planning problems and are often unable to reason about the cases where planning and perception are highly coupled during dynamic maneuvers.

\subsection{NMPC \& RL Hybrids}
Researchers have also explored NMPC-RL hybrid methods to overcome the limitations inherent to either NMPC or RL alone in scenarios where high and low-level planning are not easily decoupled. In these architectures, the RL policy often improves the long-horizon behavior, while the NMPC approach improves runtime policy performance, and in some cases, improves convergence during RL training. Many of these approaches utilize NMPC controllers as the RL policy itself. PETS \cite{pets} and POLO \cite{polo} use NMPC policies to learn a probabilistic dynamics model and a probabilistic value function, respectively. MQP \cite{mqp} integrates the action-value function into an information-theoretic NMPC objective.  Related approaches, such as DMPC \cite{dmpc}, which incorporates the value function into the stage cost, and CACTO \cite{cacto}, which leverages NMPC warm-starting in the training loop, demonstrate further convergence improvements. While several of these approaches can be applied to systems with stochastic dynamics \cite{pets,polo,mqp} and in some cases a stochastic value function \cite{polo}, they do not provide safety guarantees.

Safe RL-MPC \cite{safe-rl-mpc} provides safety guarantees of form in Def. \ref{def:safetyguarantee} by learning the parameters of a robust linear MPC controller. However, because nonlinearities must be modeled as disturbances, generated control policies are likely to be overly conservative. Another method, predictive safety filters \cite{predictive-safety-filter} modifies unsafe learned policy inputs using robust constraint-tightening NMPC, which can also lead to conservative policies \cite{kohler2020computationally}. Neither approach has been applied to systems with perception.

LOOP \cite{loop} learns both an observation dynamics model and a value function, and TD-MPC \cite{td-mpc} learns an observation dynamics model in latent space to improve sample efficiency. However, neither approach provides safety guarantees. DiffStack \cite{diffstack} and Actor Critic MPC \cite{actor-critic-mpc} embed differentiable MPC controllers directly into a learned policy. While both of these approaches can accommodate perception inputs, they do not consider stochastic dynamics or provide safety guarantees.

Table \ref{table:mpc-rl-comparison} provides a comparison of the aforementioned RL-NMPC hybrid approaches. Not only does our approach (Actor-Critic PAC-NMPC) reason about both stochastic dynamics and a stochastic value function, but it also provides probabilistic safety guarantees and can accommodate sizable perception input spaces. In addition, we evaluate our approach through hardware experiments and demonstrate the method's applicability to a challenging planning and control task.  

\begin{table}[t]
\centering
\footnotesize
\setlength{\tabcolsep}{2pt}
\renewcommand{\arraystretch}{1.08}

\begin{tabular}{|>{\raggedright\arraybackslash}m{1.64cm} ||
>{\centering\arraybackslash}m{1.14cm}|
>{\centering\arraybackslash}m{1.36cm}|
>{\centering\arraybackslash}m{1.25cm}|
>{\centering\arraybackslash}m{1.25cm}|
>{\centering\arraybackslash}m{1.09cm}|}

\hline
\centering\arraybackslash\makecell[c]{Approach}
& \makecell[c]{Stochastic\\MPC}
& \makecell[c]{Stochastic\\Value Func.}
& \makecell[c]{Safety\\Guarantees}
& \makecell[c]{Perception\\(max dim)}
& \makecell[c]{Hardware\\Demo} \\

\hline
\hline

DMPC\cite{dmpc}, CACTO\cite{cacto}
& \xmark & \xmark & \xmark & \centering\arraybackslash \xmark & \xmark \\
\hline

PETS\cite{pets}, MQP\cite{mqp}
& \cmark & \xmark & \xmark & \centering\arraybackslash \xmark & \xmark \\
\hline

DiffStack\cite{diffstack}
& \xmark & \xmark & \xmark
& \centering\arraybackslash \cmark (5$\times$a)$^*$
& \xmark \\
\hline

AC-MPC\cite{actor-critic-mpc-tro}
& \xmark & \xmark & \xmark
& \centering\arraybackslash \cmark (12$\times$2)
& \cmark \\
\hline

POLO\cite{polo}
& \cmark & \cmark & \xmark
& \centering\arraybackslash \xmark
& \xmark \\
\hline

LOOP\cite{loop},\ TD-MPC\cite{td-mpc}
& \cmark & \xmark & \xmark
& \centering\arraybackslash \cmark (16,\ 84$\times$84$\times$9)
& \xmark \\
\hline

Safe RL-MPC \cite{safe-rl-mpc}, PSF\cite{predictive-safety-filter}
& \cmark & \xmark & \cmark
& \centering\arraybackslash \xmark
& \xmark \\
\hline
\hline

AC-PAC-NMPC (Ours)
& \cmark & \cmark & \cmark
& \centering\arraybackslash \cmark (16$\times$12)
& \cmark \\
\hline

\end{tabular}

\vspace{1mm}

\begin{minipage}{\columnwidth}
\scriptsize
$^{*}$ DiffStack uses a 4-D observation and ego indicator for each agent in the scene.
\end{minipage}

\vspace{1mm}
\caption{Comparison of MPC \& RL hybrid approaches.}
\label{table:mpc-rl-comparison}
\end{table}


\subsection{Vision-based Agile Fixed-Wing Flight}

To evaluate our approach on a highly dynamic system with local perception and complex dynamics, we consider vision-based navigation through an unknown obstacle field with an agile fixed-wing aerial vehicle. Here, we briefly review some of the prior approaches explored to address this problem.

Trajectory libraries have been one method for planning and control of fixed-wing robots \cite{moore2014robust, majumdar2017funnel, barry2018}. Because these libraries are constructed offline, many of the computational challenges associated with generating trajectories for aerial robots, especially across the full flight envelope can be avoided. These approaches have also been more tightly integrated with motion planning approaches to improve performance and enable long-range planning \cite{levin2017agile,levin2019real,bulka2019high}.   

Online trajectory optimization and NMPC are another class of methods for controlling fixed-wing robots. At low angles-of-attack, a number of researchers have leveraged differentially flat representations to achieve real-time trajectory optimization \cite{6915125,bry2015,morando2025trajectory}. To achieve NMPC over a larger flight envelope, researchers have used nonlinear optimization approaches like sequential quadratic programming (SQP) \cite{9196724,agile-fixed-wing, 11127916}, MPPI \cite{pravitra2021flying}, or PAC-NMPC \cite{pacnmpc}. Researchers have also explored differentially flat representations for tail-sitter UAVs that can support aerobatic flight \cite{tal2023aerobatic}.

Fewer approaches have explored perception-based navigation with fixed-wing UAVs, since these approaches often require leveraging light-weight stereo depth cameras with limited fields-of-view which are rigidly mounted to the aircraft body and thus tightly coupled to aircraft motion. Vision-based fixed-wing navigation and obstacle avoidance has been demonstrated with trajectory libraries for low angle-of-attack reactive obstacle avoidance \cite{barry2018, bulka2022reactive} and using NMPC for agile flight in urban environments \cite{9812099, agile-fixed-wing}. In \cite{pacbayes}, researchers used reinforcement learning to generate robust policies for vision-based navigation with a tail-sitter UAV.

In this paper, we demonstrate through both simulation and hardware experiments that our hybrid RL-MPC approach can generate probabilistically safe policies for agile fixed-wing flight and enable probabilistically-safe perception-based obstacle avoidance and navigation.


\section{Problem Formulation}

In this paper, we design a controller capable of navigating a robotic system through an unknown, unstructured environment using on-board perception. The robot must reach the goal, $\goalstate$, while avoiding collisions with partially observed obstacles. We represent the evolution of the robot's state over time as a stochastic dynamical system.
\begin{definition}[Stochastic Dynamics]
    Given states $\state_{t+1}$, $\state_t \in \mathcal{X}$ and input $\inp_t \in \mathcal{U}$, the stochastic dynamics is given as $p(\state_{t+1} \mid \state_t, \inp_t)$, where $t$ is the current time index. 
    \label{def:dynamics}
\end{definition}
Throughout this paper, we assume that $p(\state_{t+1} \mid \state_t, \inp_t)$ depends continuously on $\state_t$ and $\inp_t$.

We assume that the robot has a perception sensor that provides partial observations of obstacles in the environment.
\begin{definition}[Sensor Measurement]
    Let $\obst \in \mathcal{O}$ define the obstacles in the environment. Then, the sensor measurement is $\sensor_t = \sensorfunc(\state_t, \obst)$ where $\sensorfunc:\mathcal{X}\times\mathcal{O} \mapsto \mathcal{Y}$ and $\sensor_t\in \mathcal{Y}$ .
    \label{def:sensor}
\end{definition}
Together, the state and sensor measurement represent for the robot's observation of the world.
\begin{definition}[Observation]
    The observation is given as $\obs_t = (\state_t, \sensor_t)$ where $\obs_t \in \mathcal{X} \times \mathcal{Y}$.
    \label{def:observation}
\end{definition}
We define a stage cost and constraint to define the desired navigation behavior of the robot. Specifically, the cost encourages progress towards the goal, $\goalstate$, while discouraging undesired behaviors. The constraint, dependent on a history of sensor measurements, defines unsafe states and inputs.
\begin{definition}[Stage Cost]
    The stage cost, $\costfunc(\state_t, \inp_t) \succ 0$, is a continuous function given as $\costfunc:\mathcal{X}\times\mathcal{U} \mapsto \mathbb{R}$.
    \label{def:cost}
\end{definition}
\begin{definition}[Stage Constraint]
    The stage constraint, $\constraintfunc(\state_t, \inp_t, \sensor_{(t-\sensorhist):t})$, is a continuous function given as $\constraintfunc:\mathcal{X}\times\mathcal{U}\times\mathcal{Y}^{\sensorhist} \mapsto \mathbb{R}$ where $\constraintfunc(\state_t, \inp_t, \sensor_{(t-\sensorhist):t}) \leq 0$ indicates constraint satisfaction.
    \label{def:constraint}
\end{definition}
We also define a probabilistic safety guarantee as follows:
\begin{definition}[Probabilistic Safety Guarantee]
    Let $\mathcal{S}\subset \mathcal{X}$ be a safe-set. A probabilistic safety guarantee implies
    \begin{align*}
    \mathbb{P}\big(\mathbb{P}\left(\state_{t+i} \in \mathcal{S}, \forall i \in\{0,\ldots,\horizon\} | \state_t\right)\ge 1-\beta_{\mathcal{S}} \big)\ge 1-\delta_{\mathcal{S}}
    \end{align*}
    where $\beta_{\mathcal{S}}, \delta_{\mathcal{S}}$ are small positive values, and $\horizon$ is the horizon length.
    \label{def:safetyguarantee}
\end{definition}
For Def. \ref{def:safetyguarantee}, $\delta_{\mathcal{S}} = 0$ corresponds to Safety Level II in \cite{brunke-safe-rl-survey}.
\section{Background}
Here we review PAC-NMPC \cite{pacnmpc}, a receding-horizon sampling-based SNMPC algorithm, and actor-critic RL, which comprise the main components of our approach. 
\subsection{PAC-NMPC}
PAC-NMPC is a sampling-based SNMPC approach that optimizes and provides probabilistic guarantees on both the expected cost and probability of constraint violation. It provides these guarantees in the form of PAC bounds calculated from sampled policies. Unlike many NMPC approaches, PAC-NMPC optimizes a distribution over feedback policies rather than an open-loop sequence of control actions.

Consider a state and input trajectory, $\traj = (\state_t, \inp_t, \state_{t+1}, \inp_{t+1}, \cdots, \state_{t+\horizon})$ where $\horizon$ is the trajectory horizon and $\traj \in \mathcal{T} \subset \mathcal{X}^{\horizon+1} \times \mathcal{U}^{\horizon}$. PAC-NMPC defines a trajectory cost and trajectory constraint violation indicator.
\begin{definition}[Trajectory Cost]
    The trajectory cost is given as $J: \mathcal{T} \mapsto \mathbb{R}_{\ge 0}$ such that
    \begin{equation}
        J(\traj) = \sum_{i=t}^{t+\horizon-1}\costfunc(\state_i, \inp_i)+\ell_f(\state_{t+\horizon}),
    \end{equation}
    where $\ell_f(\state_{t+\horizon})$ is the terminal cost.
    \label{def:traj_cost}
\end{definition}
\begin{definition}[Trajectory Constraint Violation Indicator]
    The trajectory constraint violation indicator is given as $C: \mathcal{T} \mapsto \{0, 1\}$ such that
    \begin{equation}
        C(\traj) = \bigvee_{i=t}^{t+\horizon-1} \left( \constraintfunc(\state_i,\inp_i,\sensor_{(t-\sensorhist):t}) > 0 \right)
    \end{equation}
    where $C(\traj) = 1$ indicates constraint violation.
    \label{def:traj_constraint}
\end{definition}
PAC-NMPC optimizes local feedback policies, $\inp_t = \bpi_t(\state_t, \bU)$, which are parameterized by open-loop nominal input trajectories, $\bU = (\inp_t^d, \cdots, \inp_{t+\horizon-1}^d)$. To generate these policies, it uses the time-varying linear quadratic regulator (TVLQR) \cite{tvlqr}. The nominal input trajectories are rolled out into state space from an initial state using a nominal, deterministic dynamics model, $\state_{t+1}^d = f(\state_t^d, \inp_t^d)$, resulting in nominal state and input trajectories, $\traj^d$. Time-varying feedback gains, ${\bf K}_t$, are computed from $\traj^d$ using TVLQR. The resulting feedback policy is $\inp_t = {\bf K}_t(\state_t^d - \state_t) + \inp_t^d$.

To optimize these feedback policies, PAC-NMPC defines a Gaussian surrogate exploration distribution, $\mathcal{N}(\bU \mid \boldsymbol{\mu}, \boldsymbol{\Sigma})$, over the nominal input trajectory. The decision variables of the optimization problem are the mean and variance of the surrogate distribution, $\bnu$ = ($\boldsymbol{\mu}$, \text{diag}($\boldsymbol{\Sigma})$). This surrogate distribution is iteratively optimized to minimize the upper PAC bounds, $\mathcal{J}^+_\alpha(\bnu)$ and $\mathcal{C}^+_\alpha(\bnu)$, over the expected trajectory cost and probability of constraint violation:
\begin{alignat}{2}
    \hat{\bnu}^*_{i+1} &=&& \argmin_{\bnu}\min_{\alpha}(\mathcal{J}_\alpha^+(\bnu)+\gamma \mathcal{C}_\alpha^+(\bnu)) \\
    & \text{s.t.} && \hspace{5mm}\mathbb{P}\left(\mathbb{E}\left[J(\traj)\right] \leq \mathcal{J}^+_\alpha(\bnu)\right) \geq 1 - \delta \nonumber\\
    & && \hspace{5mm}\mathbb{P}\left(\mathbb{E}\left[C(\traj)\right] \leq \mathcal{C}^+_\alpha(\bnu)\right) \geq 1 - \delta\nonumber
\end{alignat}
where $\delta$ is a user-defined confidence parameter. The resulting optimized surrogate distribution provides probabilistic guarantees, $\mathcal{J}^+_\alpha(\bnu)$ and $\mathcal{C}^+_\alpha(\bnu)$. Since $C(\traj)$ is a binary indicator, $\mathbb{E}\left[C(\traj)\right]$ represents the probability of constraint violation, resulting in a probabilistic safety guarantee of the same form as Definition \ref{def:safetyguarantee}. The controller executes the feedback policy constructed from the maximum likelihood estimate, $\boldsymbol{\mu}$.

The PAC bounds take the form
\begin{equation}
\mathcal{J}_\alpha^+(\bnu) = \widehat{\mathcal{J}}_\alpha(\bnu) + \alpha d(\bnu) +  \Phi_\alpha 
\end{equation}
where $\widehat{\mathcal{J}}_\alpha$ is a robust estimator of the expected cost, $d$ is a distance term from prior sampled distributions, $\Phi_\alpha$ is a concentration-of-measure term, and $\alpha$ is an annealing and regularizing coefficient. These bounds can be calculated from a finite number of policy samples.

Given $L$ prior surrogate distributions, $\bnu_0, \cdots, \bnu_{L-1}$, and $M$ iid samples from each, $(\traj_{i0}, \bU_{i0}), \cdots, (\traj_{iM-1}, \bU_{iM-1})$, the robust estimate of the expectation is given as
\begin{align}
\widehat {\mathcal J}_\alpha(\bnu) &= \frac{1}{\alpha LM} \sum_{i=0}^{L-1} \sum_{j=0}^{M-1} \psi\left(\alpha \ell_{ij} \right) \\
\psi(x) &= \log\left(1+x+\frac{1}{2}x^2\right) \\
\ell_{ij} &= J(\traj_{ij})\frac{p(\bU_{ij}|\bnu)}{p(\bU_{ij}|\bnu_i)}.
\end{align}

The distance term is given as 
\begin{align}
  d(\bnu)\! &= \! \frac{1}{2L}\sum_{i=0}^{L-1}b_i^2e^{D_2\left(p(\cdot | \bnu)||(p(\cdot | \bnu_i)\right)} \\
  0 &\leq J(\traj_{ij}) \leq b_i \; \forall j = 0, ..., M - 1 \label{eq:divergence_metric}
\end{align}
where $D_2$ is the Renyi divergence.

The concentration-of-measure term is given as 
\begin{equation}
  \Phi_{\alpha}\! = \frac{1}{\alpha LM}\log\frac{1}{\delta}.
\end{equation}

These bounds are local to each finite-horizon policy and conditional on current observations. Thus, they provide finite-time guarantees and do not guarantee global safety for closed-loop execution of the controller over the infinite horizon. 


\subsection{Actor-Critic RL}

Reinforcement learning algorithms usually involve the estimation of the value function, which represents the minimum expected accumulation of future costs. In our case, the value function can be interpreted as the optimal cost-to-go from the current state, $\state_t$, to the goal state, $\goalstate$. For the following definitions, we let $\state_t$ denote the full state of the environment, rather than the state of the robot alone.
\begin{definition}[Optimal Value Function]
    Given some stage cost $\rlcostfunc(\state_t, \inp_t)$ and stochastic dynamics $p(\state_{t+1} \mid \state_t, \inp_t)$, the optimal value function is given as $V: \mathcal{X} \mapsto \mathbb{R}$ such that
    \begin{equation} 
        V(\state_t) = \min_{\inp_t \in \mathcal{U}} \bigg[ \rlcostfunc(\state_t,\inp_t)+ \beta \EV \bigg ]
        \label{eq:sto_optimal_value_function}
    \end{equation}
    where $\beta \in (0, 1)$ is a discount factor.
    \label{def:sto_costtogo}
\end{definition}
\begin{definition}[Optimal Policy] The optimal policy is given as $\bpi^*: \mathcal{X} \mapsto \mathcal{U}$ such that
    \begin{equation}
        \bpi^*(\state_t) = \arg\min_{\inp_t \in \mathcal{U}} \bigg[ \rlcostfunc(\state_t, \inp_t)+\beta \mathbb{E}[V(\state_{t+1}) \mid \state_{t}, \inp_{t} ]\bigg ].
    \end{equation}
\end{definition}

Actor-critic methods often use temporal difference learning to jointly learn approximations of the optimal policy and optimal value function \cite{sutton2018reinforcement}. For most state-of-the-art off-policy methods, which benefit from improved sample efficiency, the critic, $\critic(\state_t, \inp_t)$, parameterized by $\criticparams$, approximates the optimal action-value function, $Q^*(\state_t, \inp_t) = \rlcostfunc(\state_t,\inp_t)+\beta \mathbb{E}\left[V(\state_{t+1}) \mid \state_{t}, \inp_{t} \right]$. The actor, $\actor(\state_t)$, parameterized by $\actorparams$, approximates the optimal policy, $\bpi^*(\state_t)$. The approximation of the optimal value function can be obtained by composing the actor and critic, $\approxvaluefunc(\state_t) = \critic(\state_t, \actor(\state_t))$. Popular actor-critic methods for continuous state and action spaces include PPO \cite{ppo}, SAC \cite{sac}, and TD3 \cite{td3}. 

The optimal value function and policy depend on the full environment state, which includes information beyond the robot state, such as the obstacles, $O$. In our case, since the robot can only partially observe the world, the actor and critic must rely on the robot's observation, $\obs_t$. Throughout this work, we interchangeably write $\approxvaluefunc(\obs_t)$ and $\approxvaluefunc(\state_t, \sensor_t)$ where appropriate, and likewise for $\actor$ and $\critic$. 


\section{Actor-Critic PAC-NMPC (AC-PAC-NMPC)}
\label{section:rlaugmentedpacnmpc}


We now describe our approach for combining PAC-NMPC with actor-critic RL to enable improved long-range performance while providing probabilistic safety guarantees. Although the learned actor and critic do not provide safety guarantees themselves, they can safely be incorporated into PAC-NMPC to guide the optimization.

Our approach utilizes three learned models. We train an actor, $\actor(\obs_t)$, and critic, $\critic(\obs_t, \inp_t)$, using simulated dynamics and sensor models. Additionally, we train a sensor prediction model, $\hat{\sensor}_{t+k} =\sensorpredictfunc(\state_t, \sensor_t, \state_{t+k})$, to predict future sensor measurements along sampled trajectories, thus allowing the actor and critic to be evaluated within PAC-NMPC. An overview of our approach is shown in Figure \ref{fig:block}.

We extend the PAC-NMPC algorithm in three major ways:
\begin{enumerate}
    \item RL-based warm start of the decision variables.
    \item Learned stochastic value function as terminal cost.
    \item Value function improvement constraint along trajectory.
\end{enumerate}
These contributions will be discussed in the subsequent subsections. Since the implementation and training of these learned models is specific to each robotic system, we leave discussion of these components to later sections.

\begin{figure}[t]
    \centering
    \includegraphics[trim={0 0 0 0},clip,width=1.0\columnwidth]{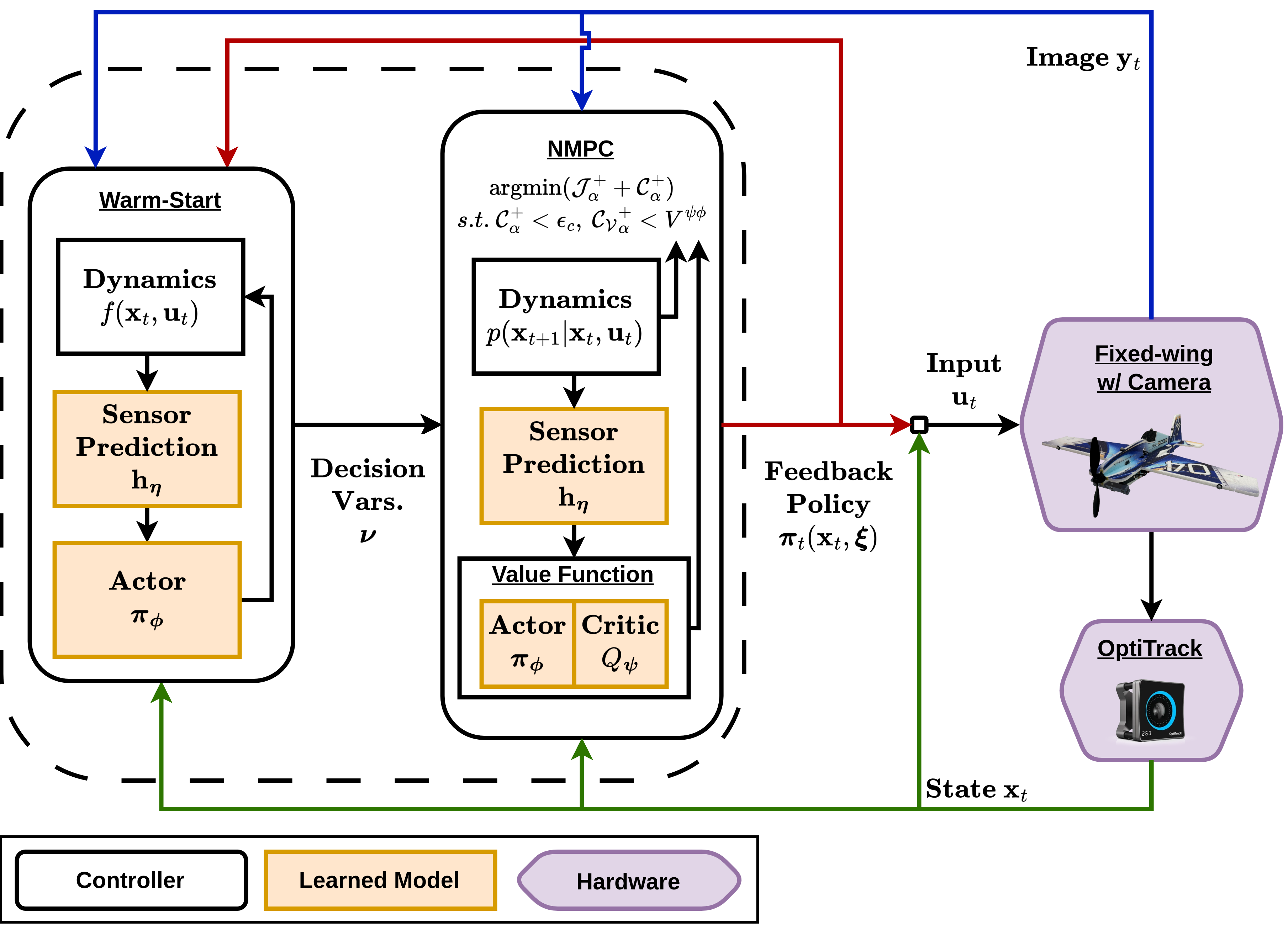}
    \caption{Overview of AC-PAC-NMPC.}
    \label{fig:block}
\end{figure}


\subsection{RL-based Warm Start}

We initialize the PAC-NMPC decision variables using the learned actor policy. Since the RL actor policy is trained to minimize the expected cost-to-go, this warm-start may reduce optimization time to convergence and help avoid non-optimal local minima. For high-dimensional input spaces, we have found that this feature is important for keeping the SNMPC policy within the distribution of the trained RL model. Our warm-start approach is summarized in Algorithm \ref{algo:warm-start}.

Given an initial state, $\state_t$, and sensor measurement, $\sensor_t$, the initial input of the warm-start trajectory can be computed as $\hat{\inp}_t = \actor(\state_t, \sensor_t)$. However, subsequent actions cannot be computed directly, since future states and sensor measurements are needed to evaluate the actor along the rest of the trajectory, $\hat{\inp}_{t+k} = \actor(\state_{t+k}, \sensorpredict_{t+k})$. To address this, we compute future states recursively using a nominal, deterministic dynamics model, $\state_{t+1} = \nominaldynamicsfunc(\state_t, \inp_t)$. 

Unlike the robot state, future sensor measurements cannot be simulated in the same way because they depend on an environment which is unknown \emph{a priori}. Instead, future sensor measurements are estimated using a sensor prediction model, $\sensorpredict_{t+k} = \sensorpredictfunc(\state_{t}, \sensor_{t}, \state_{t+k})$. Rather than predicting the sensor model recursively, $\sensorpredict_{t+1} = \sensorpredictfunc(\state_{t}, \sensor_{t}, \state_{t+1})$, we predict each future sensor measurements at time $t+k$ directly from the initial sensor measurement. Empirically, we found that recursively applying sensor prediction led to error accumulation and degraded sensor measurement prediction.

The resulting actor warm-start initializes the mean of the SNMPC policy's surrogate distribution, $\boldsymbol{\mu}$, while the variance, $\boldsymbol{\Sigma}$, is set to a prespecified high variance to promote exploration. Due to runtime limitations, it is often not possible for PAC-NMPC to fully reduce this exploration variance. However, because we ultimately execute the maximum likelihood estimate of the SNMPC policy, we can perform a final PAC-NMPC policy optimization iteration in which the variances are reduced to a prespecified small value. This produces PAC guarantees that more accurately reflect the executed policies and are often tighter.

While PAC-NMPC can often find a feasible solution with this actor warm-start, when the warm-start produces a trajectory that is severely infeasible (e.g., due to distribution shift), the PAC-NMPC policy distribution may not sample any feasible trajectories. The lack of feasible trajectories in this case will result in failed optimization. To improve out-of-distribution performance, we evaluate the degree to which the actor policy warm-start violates constraints. If it does, we instead default to warm-starting using the previously optimized PAC-NMPC policy, as done in prior research \cite{pacnmpc}. 


\begin{algorithm}[!t]
  \small
  \caption{Warm Start}\label{algo:warm-start}
  \KwIn{$\state_t$, $\sensor_{(t-\sensorhist):t}$}
  $\hat{\inp}_t = \actor(\state_t, \sensor_t)$\;
  \For{$i=t,\cdots,t+\horizon-1$}
  {
    $\state_{i+1} = \nominaldynamicsfunc(\state_i, \hat{\inp}_i)$\;
    $\sensorpredict_{i+1} = \sensorpredictfunc(\state_{t}, \sensor_{t}, \state_{i+1})$\;
    $\hat{\inp}_{i+1} = \actor(\state_{i+1}, \sensorpredict_{i+1})$\;
    \If{$\constraintfunc(\state_{i+1}, \hat{\inp}_{i+1}, \sensor_{(t-\sensorhist):t}) > 0$}
    {
        warm-start from prior policy\;
        \Return
    }
  }
  warm-start from $\boldsymbol{\mu} = (\hat{\inp}_t,\cdots,\hat{\inp}_{t+\horizon-1})$\;
\end{algorithm}


\subsection{Uncertainty-Aware RL-Augmented Cost}
A key limitation of NMPC is the potential myopic behavior from optimizing over a finite horizon. For online planning scenarios, this horizon is frequently restricted by computational resources, especially for robots characterized by complex dynamics and large state spaces. Taking inspiration from \cite{polo}, we use the learned value function as the terminal cost
\begin{equation}
\costfunc_f(\state_{t+\horizon}) = \approxvaluefunc(\state_{t+\horizon}, \sensorpredict_{t+\horizon})
\end{equation}
where $\approxvaluefunc$ is composed of the actor, $\actor$, and critic, $\critic$. A key difference between our approach and \cite{polo} is that our value function is dependent on sensor measurements and, therefore, we predict the terminal sensor measurement, $\sensorpredict_{t+\horizon} = \sensorpredictfunc(\state_t, \sensor_t, \state_{t+\horizon})$. Consequently, the resulting trajectory cost approximates the infinite-horizon perception-based cost-to-go given the current state and sensor measurements.

Also similar to \cite{polo}, we learn an approximate distribution over value functions. However, while \cite{polo} did this to facilitate exploration, we do this so that PAC-NMPC can reason not only about the uncertainty in the robot dynamics, but also the uncertainty in the value function, during run-time execution.

To model the stochastic value function, we implement $\actor$, $\critic$, and $\sensorpredictfunc$ as stochastic models by representing these models as neural networks trained with Monte Carlo (MC) dropout:
\begin{equation}
r^{(l)}_j \sim \text{Bernoulli}(p) \qquad \mathbf{\tilde{y}}^{(l)} = \mathbf{r}^{(l)} * \mathbf{y}^{(l)}
\end{equation}
where $l \in \{1,\cdots,L\}$ indexes the hidden layers, $\mathbf{r}^{(l)}$ is the vector of independent Bernoulli random variables that have a probability $p$ of being $1$, $\mathbf{y}^{(l)}$ denotes the vector of outputs from layer $l$, $*$ denotes an element-wise product, and $\mathbf{\tilde{y}}^{(l)}$ denotes the post-dropout outputs. Dropout was originally proposed as a way of reducing neural network overfitting \cite{srivastava2014dropout}, but has also been shown to approximate a Bayesian neural network when used during inference \cite{pmlr-v48-gal16}.

One traditional shortcoming of MC dropout is that the network must be inferenced many times with a different dropout mask for each forward pass, thus significantly increasing inference time. However, since PAC-NMPC must already inference these networks for each sampled trajectory, applying a different dropout mask for each trajectory results in no additional forward passes. Thus, while other approaches can be used to model uncertainty, such as ensembles and multi-headed networks, we select MC dropout since it incurs both a negligible training time and inference time cost. 

We note that, for simple (e.g., 2D) environments and sensors with large fields-of-view, a learned sensor prediction model may not be necessary for successful navigation.



\subsection{Uncertainty-Aware Value Function Improvement Constraint}
While using the learned value function as the terminal cost encourages less myopic behavior, it does not guarantee that the resulting policy will actually reduce the cost-to-go. Therefore, we also apply a value function improvement constraint along the trajectory.

We introduce an additional PAC-bound, $\mathcal{C_V}^+_\alpha(\bnu)$, on the learned value function at the terminal state such that
\begin{equation}
\mathbb{P}\left(\mathbb{E}\left[\approxvaluefunc(\state_{t+\horizon}, \sensorpredict_{t+\horizon})\right] \leq \mathcal{C_V}^+_\alpha(\bnu)\right) \geq 1 - \delta.
\end{equation}
We constrain the PAC-NMPC optimization problem such that this bound must be less than the learned value function at the initial state of the trajectory
\begin{equation}
\text{s.t.}\quad \mathcal{C_V}^+_\alpha(\bnu) \leq \mathbb{E}\left[\approxvaluefunc(\state_t, \sensor_t)\right].
\end{equation} 
Since in practice $\mathcal{C_V}^+_\alpha(\bnu)$ has a nonzero upper-bound gap such that $\mathbb{E}\left[\approxvaluefunc(\state_{t+\horizon}, \sensorpredict_{t+\horizon})-\approxvaluefunc(\state_t, \sensor_t)\right]<0$ and because this learned value function approximates the cost-to-go, satisfaction of this constraint indicates that the robot is expected to reduce the cost-to-go over the finite planning horizon. Similarly to the prior section, we sample from the stochastic learned value function to incorporate model uncertainty into this PAC bound. In contrast to prior research \cite{pacnmpc-lvf}, which bundled the value function improvement condition into the probability of constraint violation bound, this approach provides two distinct probabilistic guarantees.

The resulting optimization problem is
\begin{align}
\bnu^* &= \argmin_{\bnu}\min_{\alpha>0} (\mathcal{J}^+_\alpha(\bnu) + \gamma \mathcal{C}^+_\alpha(\bnu)) \\ \nonumber
& \text{s.t.}\quad \mathcal{C}^+_\alpha(\bnu) \leq \epsilon_c. \\ \nonumber
& \quad \quad \ \mathcal{C_V}^+_\alpha(\bnu) \leq \mathbb{E}\left[\approxvaluefunc(\state_t, \sensor_t)\right],
\end{align}
where $\epsilon_c$ is the maximum allowable probability of constraint violation. Also, in contrast to \cite{pacnmpc-lvf}, both PAC bounds are imposed as hard constraints in the optimization problem. We leave the constraint violation bound as a penalty in the optimization function itself, as it is often beneficial to further minimize the violation bound below $\epsilon_c$ whenever possible. We solve this optimization problem using SNOPT \cite{snopt}, an SQP algorithm for constrained optimization.


\section{Rally Car UGV with LIDAR: Approach and Evaluation}
\label{section:rallycar}

In this section, we discuss the application of AC-PAC-NMPC to a 1/10th scale rally car platform equipped with a LiDAR sensor (Fig. \ref{fig:rallycar}). This section follows the formulation and includes experimental results presented in \cite{pacnmpc-lvf}. We include this summary of the formulation and experimental results to provide a self-contained presentation of the approach across multiple robotic platforms and to establish a basis for comparison against the more complex system presented in Section \ref{section:fixedwing}. In particular, because this prior research does not leverage the RL warm-start and learned perception prediction components described in the prior section, it helps demonstrate how these components become more critical for dynamical systems with larger states spaces, larger measurement spaces, and a diminished field-of-view. 


\subsection{Rally Car Dynamics and LiDAR Sensor}

The rally car dynamics are represented by a stochastic bicycle model. The state is given as $\state_t=[r_x \ r_y \ \theta \ v \ \delta_s]^T$ where $r_x$, $r_y$ denote the position, $\theta$ is the orientation, $v$ the forward speed, and $\delta_s$ the steering angle. The input is given as $\inp_t=[\dot{v} \ \dot{\delta_s}]^T$ where $\dot{v}$ is the forward acceleration and $\dot{\delta_s}$ is the steering angle rate of change.


The continuous-time, nominal dynamics are given as 
\begin{equation}
\nominaldynamicsfunc(\state_t, \inp_t) = [v\cos(\theta), v\sin(\theta), \frac{v\tan(\delta_s)}{l}, \dot{v}, \dot{\delta_s}]^T
\end{equation}
where $l = 0.33 \mathrm{m}$ is the wheelbase. The stochastic, discrete-time dynamics applies Gaussian noise and Euler integration with a process covariance of $\boldsymbol{\Sigma_f} = diag(\left[4{\rm e}^{-4}, 4{\rm e}^{-4}, 1.1{\rm e}^{-2}, 1{\rm e}^{-1}, 5.6{\rm e}^{-3}\right])$, which was fit from hardware data. There are limits on the acceleration, $\dot{v}\in\left[-1., 1\right] \mathrm{\frac{m}{s^2}}$, the steering rate, $\dot{\delta_s}\in\left[-1, 1 \right] \mathrm{\frac{rad}{s}}$, and the steering angle, $\delta_s\in\left[-0.4, 0.4\right] \mathrm{rad}$.

The rally car was equipped with a 64 beam, $360^{\circ}$ planar LiDAR, which returned range measurements $\sensor_t = [y_t^0 \ y_t^1 \ ... \ y_t^{63}]$ at bearings $\bbbeta = [\beta^0 \ \beta^1 \ ... \ \beta^{63}]$.

\begin{figure}[t]
    \centering

    \begin{minipage}[t]{0.48\columnwidth}
        \centering
        \includegraphics[trim={150 100 150 30},clip,width=\linewidth]{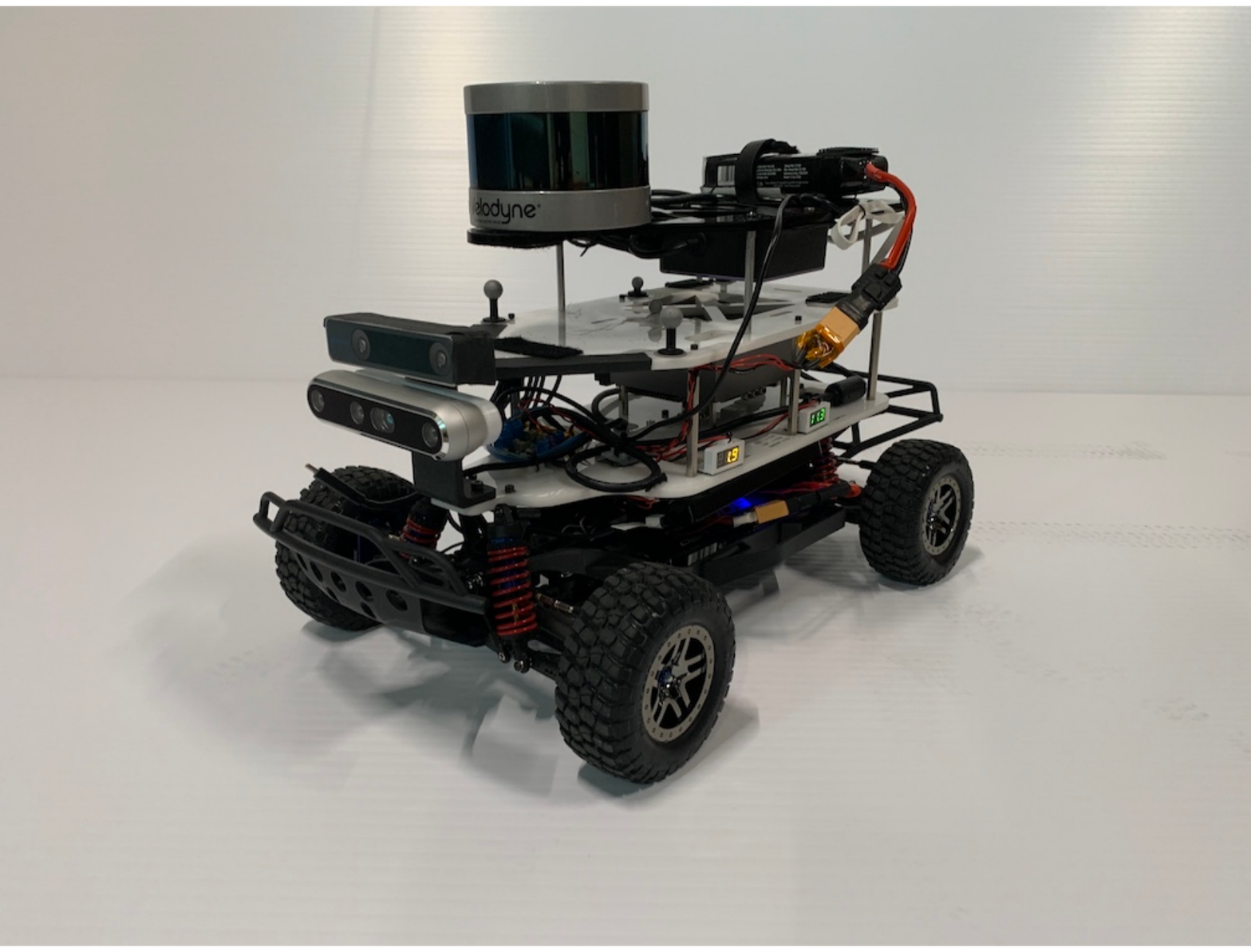}
        \captionof{figure}{1/10th scale rally car platform with LiDAR sensor.}
        \label{fig:rallycar}
    \end{minipage}\hfill
    \begin{minipage}[t]{0.48\columnwidth}
        \centering
        \includegraphics[trim={50 40 90 0},clip,width=\linewidth]{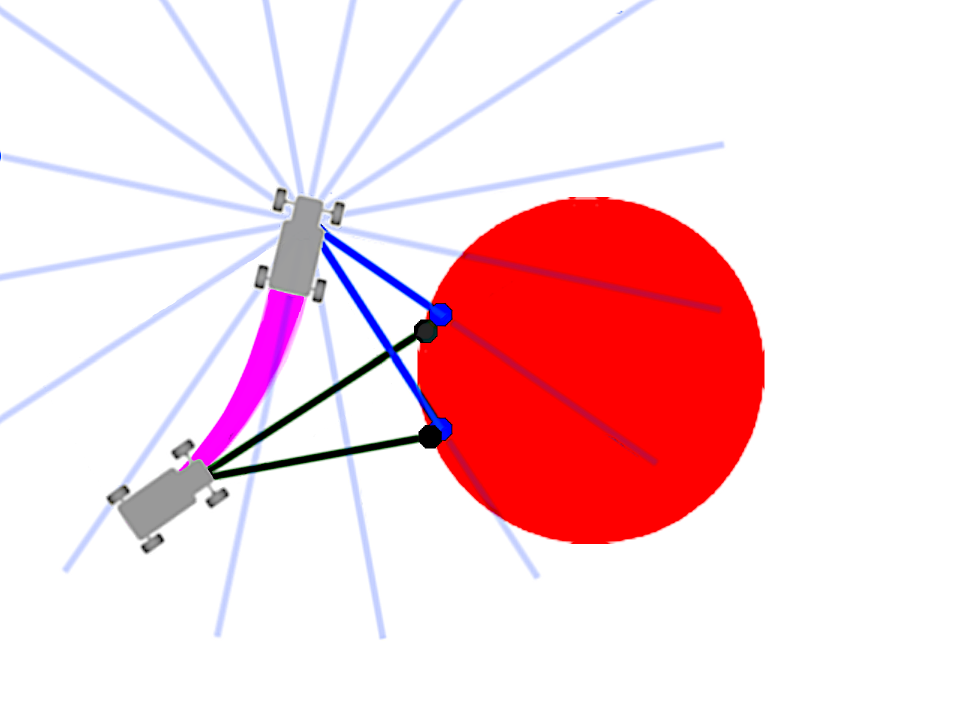}
        \captionof{figure}{Geometric LiDAR prediction.}
        \label{fig:lidar_est}
    \end{minipage}

\end{figure}


\subsection{Actor-Critic}

\label{subsection:rallycar-actor-critic}

Using the dynamics and sensor model described above, the actor and critic networks were trained using the CleanRL \cite{cleanrl} implementation of TD3 \cite{td3}. During training, the stage cost was summed with a constraint violation penalty:
\begin{equation}
\rlcostfunc(\state_t, \inp_t) = \costfunc(\state_t, \inp_t)+\gamma_c \constraintfunc(\state_t, \inp_t)
\end{equation}
where $\gamma_c = 1000$ was a heuristically selected penalty. Unlike the stage constraint used during planning (Def. \ref{def:constraint}), the training constraint, $\constraintfunc(\state_t, \inp_t) \in \{0, 1\}$, utilizes full environmental knowledge, and $\constraintfunc(\state_t, \inp_t) = 1$ indicates constraint violation.

The actor and critic, $\actor(\obs_t)$ and $\critic(\obs_t,\inp_t)$ respectively, were parameterized by fully connected neural networks with two hidden layers, 256 neurons per layer, 10\% dropout, and ReLU activation functions. To aid efficient learning, the observations were preprocessed before being input into the networks. The network inputs consisted of the speed, tangent of the steering angle, range to the goal, cosine and sine of the bearing to the goal, and the LiDAR measurement, which were all normalized. Training environments, which included obstacles, initial states, and goal states, were randomly sampled. These networks were trained until the actor loss converged.


\subsection{Sensor Prediction}

Because these experiments were restricted to a two-dimensional environment with planar dynamics and a sensor model with a 360$^o$ field-of-view, sensor prediction for this system could be performed geometrically and did not incorporate prediction uncertainty quantification. This approach is visualized in Figure \ref{fig:lidar_est}.

Given the current state, $\state_t$, each valid LiDAR return was projected into an observed obstacle point $\mathbf{o}^j$:
\begin{align}
    \mathbf{o}^j =
    \begin{bmatrix}
        {r_x}_t \\
        {r_y}_t
    \end{bmatrix}
    +
    \begin{bmatrix}
        \cos(\theta_t + \beta^j) \\
        \sin(\theta_t + \beta^j)
    \end{bmatrix}
    y_t^j.
\end{align}
Then, the range and bearing to each of these points from the future state $\state_{t+i}$ was calculated as
\begin{align}
    \hat{y}^j &= \| [\mathbf{o}^j - {r_x}_{t+i} \ {r_y}_{t+i}]^T \| \\
    \hat{\beta}^j &= \atantwo(\mathbf{o}^j - [{r_x}_{t+i} \ {r_y}_{t+i}]^T) - \theta_{t+i}.
\end{align}
Since these bearing estimates do not necessarily correspond to the discrete bearing values used by the sensor, the estimated LiDAR measurements were assigned to the closest bearing:
\begin{equation}
\hat{y}_{t+i}^k = \hat{y}^j \qquad k = \argmin_{k}\{|\hat{\beta}^j-\beta^k|\}.
\end{equation}


\subsection{Simulation Experiments}

AC-PAC-NMPC was compared against four baselines across two simulation environment distributions. Unique actor and critic networks were trained for each distribution. RL training and simulations were run on a laptop with an Intel Core i-9-13900H CPU and a Nvidia GeForce RTX 4080 Max-Q GPU.

The first baseline used PAC-NMPC with a quadratic terminal cost, $\costfunc_f(\state_{t+\horizon}) = (\state_{t+\horizon}-\goalstate)^T \mathbf{Q}_f (\state_{t+\horizon}-\goalstate)$ where $\mathbf{Q}_f = diag([1 \ 1 \ 0 \ 0 \ 0])$. In the second baseline, LiDAR measurements were used to continuously build and maintain an occupancy grid of the environment using the widely adopted Nav2 framework \cite{nav2}. Each planning iteration, the A* search algorithm was run on the grid to find the shortest path to the goal. A receding horizon goal, $\state_{A^*}$, was selected at a distance of $v_{max} \cdot \horizon \cdot \Delta t$ along the path and was used to form a quadratic terminal cost, $\costfunc_f(\state_{t+\horizon}) = (\state_{t+\horizon}-\state_{A^*})^T \mathbf{Q}_f (\state_{t+\horizon}-\state_{A^*})$. The third baseline evaluated the actor policy itself, $\bpi^{\phi}(\state)$. The fourth baseline optimized trajectories with MPPI \cite{mppi}, while still using the learned value function as a terminal trajectory cost.

The following experiments used a quadratic stage cost, $\costfunc(\state_t, \inp_t) = (\state_t-\goalstate)^T \mathbf{Q} (\state_t-\goalstate)$, where $\mathbf{Q} = diag([1{\rm e}^{-2} \ 1{\rm e}^{-2} \ 0 \ 0 \ 0])$. The stage constraint, $\constraintfunc(\state_t,\inp_t,\sensor_t)\le0$, enforced bounds on the forward speed, $v \in \left[-1,3\right] \mathrm{\frac{m}{s}}$, and required the vehicle to maintain a minimum distance of $0.5 \mathrm{m}$ from all observed points from the latest LiDAR measurement. RL-based warm-starting was not used in these experiments.

PAC-NMPC optimized feedback policies over 12 timestep trajectories with $\Delta t=0.1$ sec at a replanning period of $H=0.2$ sec. The optimization used $L=5$ prior policies with $M=1024$ trajectory samples per prior  with $\delta=0.05$. The trajectory costs were normalized before optimization to achieve tighter PAC bounds. The constraint violation bound penalty was set to $\gamma=2$ in the cluttered environments and $\gamma=4$ in the concave trap environments. These were empirically determined to yield the best performance in these environments. When running MPPI, the sampled trajectory costs were similarly normalized and summed with a constraint violation indicator multiplied by $\gamma$. The same number of timesteps, $\Delta t$, $H$, and $M$ were used, along with a temperature of $\gamma_t=0.35$ and sampling variance of $\Sigma_\epsilon=0.01$. Both PAC-NMPC and MPPI were allowed to run for as many iterations as possible within the replanning period and the resulting policies were interpolated to 50Hz.

\subsubsection{Cluttered Environments}

The first distribution of environments consisted of randomly placed circular obstacles with random radii. These obstacles were allowed to overlap, which allowed the constraint regions to combine to form complex environments. 100 testing environments were sampled from this distribution. Trivial environments in which no obstacles blocked the path to the goal were discarded.

\begin{table}[t]
\centering
\begin{tabular}{|c | c || c| c| c|}
\hline
\multicolumn{2}{|c||}{Approach} & Success & Stuck & Violation \\
\hline
\hline
\multicolumn{2}{|c||}{RL Actor Network} & $90\%$ & $\mathbf{3\%}$ & $7\%$  \\ 
\hline
\hline
\multicolumn{2}{|c||}{MPPI w/ Learned $\approxvaluefunc$} & $70\%$ & $7\%$ & $23\%$  \\ 
\hline
\hline
\multirow{3}{*}{PAC-NMPC} & Quad. Term. Cost & $76\%$ & $24\%$ & $\mathbf{0\%}$ \\
\cline{2-5}
 & Map \& A* & 89\% & 4\% & 7\% \\
\cline{2-5}
 & Learned $\approxvaluefunc$ & $\mathbf{97\%}$ & $\mathbf{3\%}$ & $\mathbf{0\%}$ \\
\hline
\end{tabular}
\vspace{2mm}
\caption{Simulated cluttered environments results.}
\label{table:cluttered-sim}
\end{table}

\begin{figure}[t]
    \centering
    \includegraphics[trim={185 70 180 120},clip,width=1.0\columnwidth]{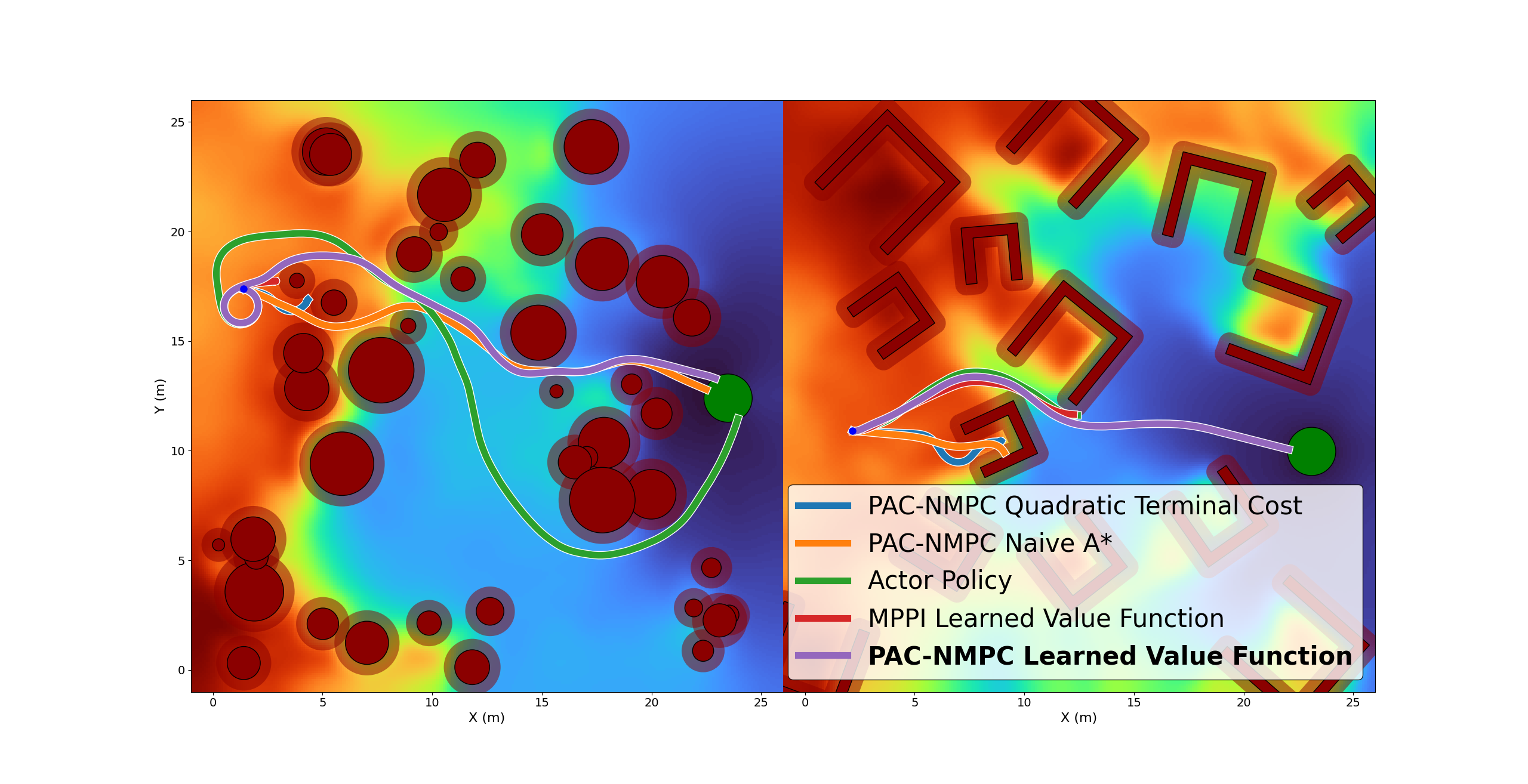}
    \caption{Cluttered and trap environment examples. Slices of the learned value functions are visualized as heatmaps.}
    \label{fig:rallycar-sim-example}
\end{figure}

Results are shown in Table \ref{table:cluttered-sim}, with an example environment in Figure \ref{fig:rallycar-sim-example}. When using a quadratic terminal cost, the system often got caught in local minima. When using A*, the system occasionally planned paths through obscured obstacles, which caused constraint violations if the LiDAR was unable to view the obscured obstacles until the system was too close to recover. The actor policy was unable to explicitly enforce constraints which resulted in occasional violations. Our approach outperformed all baselines and never violated the constraints.

\begin{figure}[t]
    \centering
    \includegraphics[trim={70 10 90 80},clip,width=0.8\columnwidth]{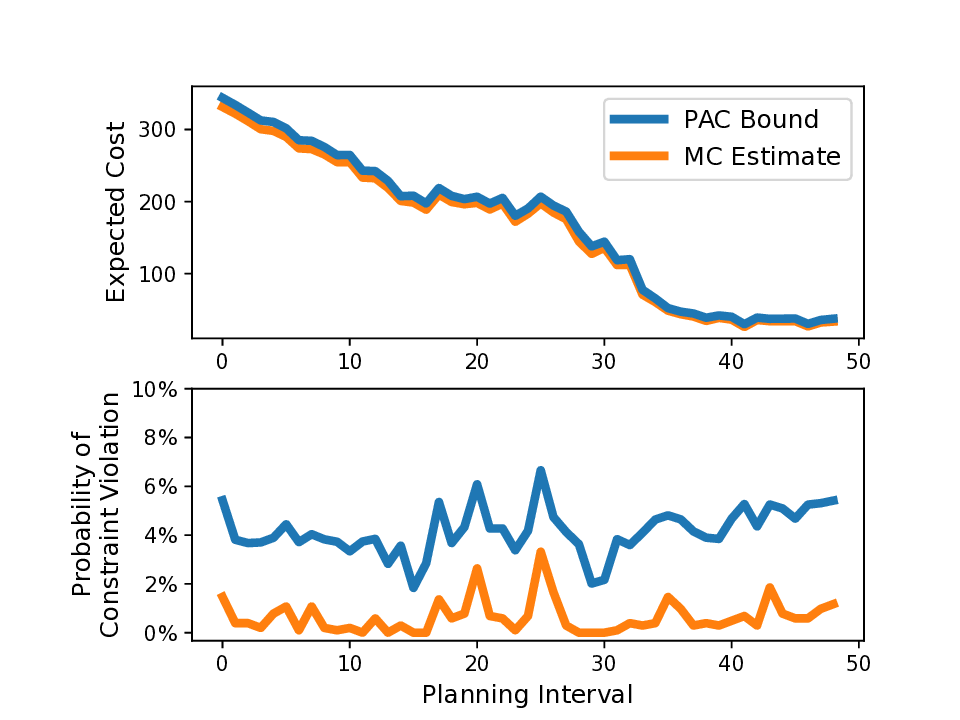}
    \caption{PAC Bounds \& Monte Carlo estimates.}
    \label{fig:pacbounds}
\end{figure}

Figure \ref{fig:pacbounds} displays a plot of the optimized PAC bounds at each planning interval for an entire example trial when using the learned value function as a terminal cost and constraint. The probability of collision remains non-zero throughout the trial due to dynamics uncertainty and increases as the system gets closer to obstacles. To validate PAC bounds, the bounds were compared against Monte Carlo estimates of the expected cost and probability of constraint violation, which were formed by sampling 1024 trajectories and dropout masks. On average, PAC-NMPC produced guarantees that the probability of constraint violation of the NMPC generated policies would be less than 5\%. 

To demonstrate that this approach incorporated uncertainty from the actor and critic networks into the PAC bound computation, the percentage of bound violations was evaluated when optimizing the bounds with and without sampling dropout masks. In both cases, these were compared against Monte Carlo estimates using 1024 sampled trajectories and dropout masks. Across all trials, when optimizing with sampled dropout masks, the expected cost bound was never violated, and the probability of constraint violation bound was violated in only 0.34\% of planning intervals. When optimizing without sampled dropout masks, the expected cost bound was violated in 65.49\% of planning intervals, and the probability of constraint violation bound was violated in 1.45\%.

\subsubsection{Concave Trap Environments}

\begin{table}[t]
\centering
\begin{tabular}{|c | c || c| c| c|}
\hline
\multicolumn{2}{|c||}{Approach} & Success & Stuck & Violation \\
\hline
\hline
\multicolumn{2}{|c||}{RL Actor Network} & $82\%$ & $\mathbf{2\%}$ & $16\%$  \\ 
\hline
\hline
\multicolumn{2}{|c||}{MPPI w/ Learned $\approxvaluefunc$} & $55\%$ & $9\%$ & $36\%$  \\ 
\hline
\hline
\multirow{3}{*}{PAC-NMPC} & Quad. Term. Cost & $47\%$ & $53\%$ & $\mathbf{0\%}$ \\
\cline{2-5}
 & Map \& A* & 91\% & 4\% & 5\% \\
\cline{2-5}
 & Learned $\approxvaluefunc$ & $\mathbf{93\%}$ & $7\%$ & $\mathbf{0\%}$ \\
\hline
\end{tabular}
\vspace{2mm}
\caption{Simulated trap environments results.}
\label{table:trap-sim}
\end{table}

The second distribution of environments consisted of concave traps to highlight the ability of AC-PAC-NMPC to safely avoid local minima. The number of obstacles, side lengths of each obstacle, and pose of each obstacle was randomly sampled. The angles of the obstacles were sampled uniformly such that they were pointing towards the starting position of the robot $\pm\frac{\pi}{2}$ radians. 100 environments were sampled from the distribution and used to evaluate AC-PAC-NMPC, which outperformed all baselines and never violated the constraints (Table \ref{table:trap-sim}). An example environment is shown in Figure \ref{fig:rallycar-sim-example}.


\subsection{Hardware experiments}

\begin{figure}[t]
\centering
\includegraphics[trim={0 0 0 0},clip,width=\columnwidth]{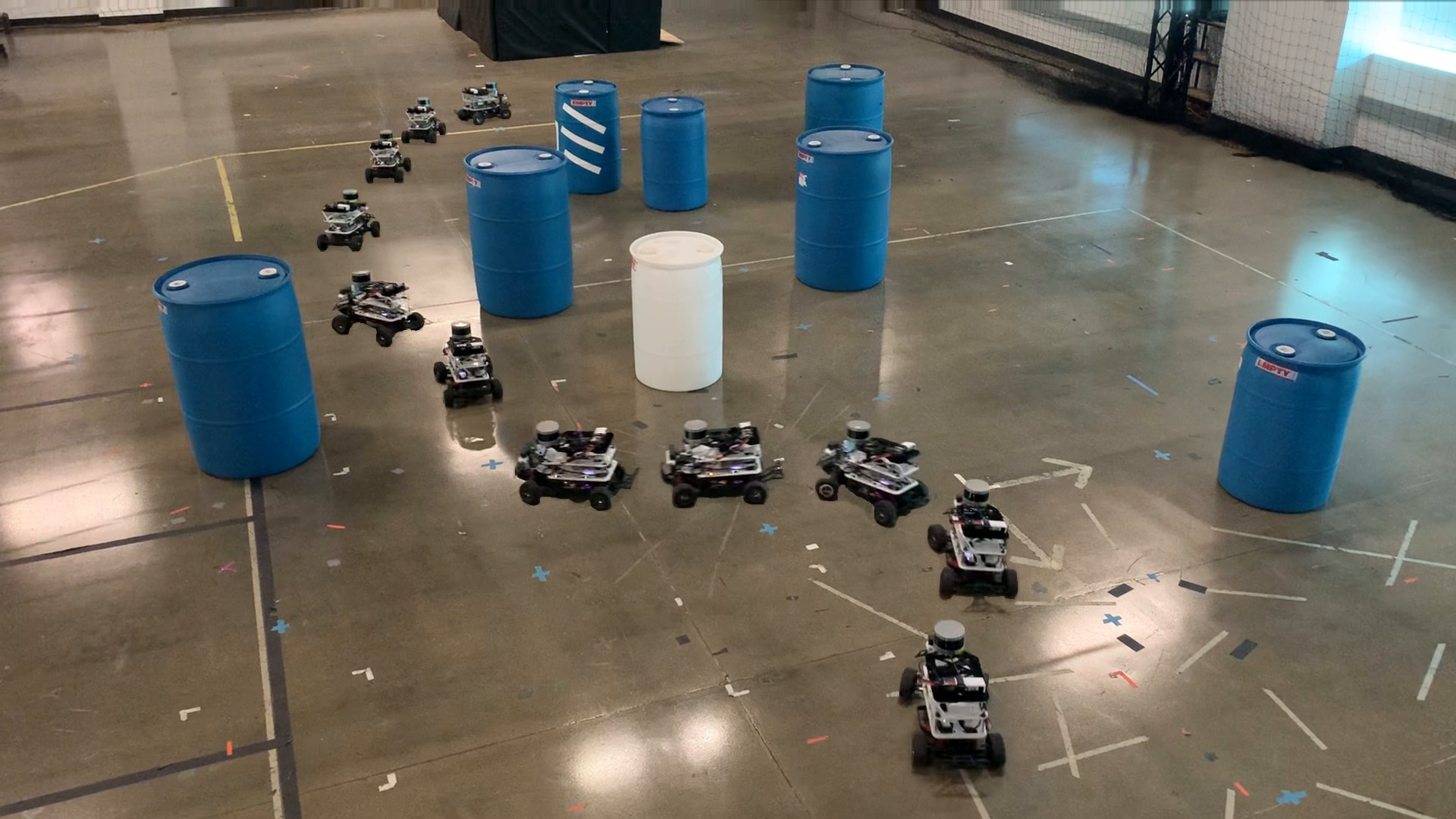}
\caption{Hardware environment example.}
\label{fig:rallycar-hardware-env}
\end{figure}

\begin{table}[t]
\setlength{\tabcolsep}{3pt} 
\centering
\begin{tabular}{|c || c| c| c|}
\hline
Approach & Success & Stuck & Violation \\
\hline
\hline
Actor Policy & $85\%$ & $\mathbf{0\%}$ & $15\%$ \\
\hline
PAC-NMPC w/ Learned $\approxvaluefunc$ & $\mathbf{95\%}$ & $5\%$ & $\mathbf{0\%}$ \\
\hline
\hline
Actor Policy w/ Mismatch & $70\%$ & $\mathbf{0\%}$ & $30\%$ \\
\hline
PAC-NMPC w/ Learned $\approxvaluefunc$ w/ Mismatch & $80\%$ & $20\%$ & $\mathbf{0\%}$ \\
\hline
\end{tabular}
\vspace{2mm}
\caption{Hardware environments results.}
\label{table:rallycar-hardware-results}
\end{table}

AC-PAC-NMPC was evaluated on a 1/10\textsuperscript{th} scale Traxxas Rally Car platform with a Velodyne Puck LITE LiDAR. An OptiTrack Motion Capture system was used for state estimation. 20 random environments consisting of a maximum of 8 circular obstacles in an 8 meter by 6 meter space were generated. Environments in which obstacles overlapped or in which no obstacles were blocking the path to the goal were discarded. An example environment is shown in Figure \ref{fig:rallycar-hardware-env}.

The value function used in these hardware experiments was trained entirely in the cluttered environment simulation. In contrast to the simulated LiDAR, the Velodyne Puck LITE LiDAR has noise and returns range measurements to the walls. Thus, these environments are outside of the training distribution. Additionally, a value function was also trained with an incorrect wheelbase, $0.5 \mathrm{m}$, to demonstrate the utility of this approach even in the presence of model mismatch.

AC-PAC-NMPC outperformed the actor policy and never collided with obstacles (Table \ref{table:rallycar-hardware-results}). This indicates that utilizing the RL models inside PAC-NMPC provided better robustness for the sim-to-real transfer.
When using the learned value function trained with the incorrect wheelbase, but using the correct wheelbase when sampling trajectories, AC-PAC-NMPC outperformed the actor policy and never collided with obstacles. Thus, it was able to safely utilize an RL model trained in the presence of model mismatch.


\section{Fixed-Wing UAV with Depth Camera: Approach and Evaluation}
\label{section:fixedwing}

In this section, we apply AC-PAC-NMPC to a 32-inch wing-span fixed-wing aerial vehicle equipped with a depth camera (Fig. \ref{fig:fixedwing}). This system has significantly more complex dynamics and sensing than the rally car in Section \ref{section:rallycar}. Consequently, the novel algorithmic contributions presented in Section \ref{section:rlaugmentedpacnmpc}, including the RL-based warm start, learned sensor prediction model, and separate value function improvement PAC bound, are critical for enabling the successful navigation.


\subsection{Fixed-wing Dynamics and Depth Camera Sensor}

We use the dynamics model proposed in \cite{agile-fixed-wing} with a nonunit-quaternion orientation representation \cite{nonunit-quaternions}.

The state and control are defined as 
\begin{align}
    \state_t = \left[ \mathbf{r}^T \ \mathbf{q}^T \ \mathbf{v}^T \ \boldsymbol{\omega}^T \ \boldsymbol{\delta_s}^T \ \delta_{th} \right]^T \quad \inp_t = \left[ \mathbf{u_s}^T \ u_{th} \right]^T
\end{align}
where $\mathbf{r} \in \mathbb{R}^3$ is the position of the center of mass in the world frame, $\mathbf{q} \in \mathbb{H}$ is the orientation, $\mathbf{v} \in \mathbb{R}^3$ is the  linear velocity in the world frame, $\boldsymbol{\omega} \in \mathbb{R}^3$ is the angular velocity in the body frame, $\boldsymbol{\delta_s} = \left[\delta_a, \delta_e, \delta_r \right] \in \mathcal{D} \subset \mathbb{R}^3$ are the control surface angles, $\delta_{th} \in \mathbb{R}$ is the propeller thrust magnitude, $ \mathbf{u_s} \in \mathbb{R}^3$ are the control surface angular rates, and $u_{th} \in \mathbb{R}$ is the throttle. The continuous-time, nominal dynamics are defined as
{
\setlength{\belowdisplayskip}{0pt}
\setlength{\belowdisplayshortskip}{0pt}
\begin{align}
\dot{\state}_t = \nominaldynamicsfunc(\state_t, \inp_t) = \left[ \dot{\mathbf{r}}^T \ \dot{\mathbf{q}}^T \ \dot{\mathbf{v}}^T \ \dot{\boldsymbol{\omega}}^T \ \dot{\boldsymbol{\delta}}_s^T \ \dot{\delta}_{th} \right]^T \nonumber
\end{align}
}
{
\setlength{\abovedisplayskip}{0pt}
\setlength{\abovedisplayshortskip}{0pt}
\begin{align}
\mathbf{\dot{r}}
&= \mathbf{v}
&
\boldsymbol{\dot{\omega}}
&= \mathbf{J}^{-1}(\mathbf{m}-\boldsymbol{\omega}\times \mathbf{J}\boldsymbol{\omega})
\notag\\
\mathbf{\dot{q}}
&= \frac{1}{2}\boldsymbol{\Omega}(\boldsymbol{\omega})\mathbf{q}
+ 0.1(1-\mathbf{q}^T\mathbf{q})\mathbf{q}
&
\boldsymbol{\dot{\delta}_s}
&= \mathbf{u_s}
\notag\\
\mathbf{\dot{v}}
&= \mathbf{R}(\boldsymbol{\mathbf{q}})\mathbf{f}/m
&
\dot{\delta}_{th}
&= a\delta_{th}+bu_{th}
\end{align}
}
where $\mathbf{R}(\boldsymbol{\mathbf{q}}) \in SO(3)$ is the rotation matrix, $m$ is the mass, $\mathbf{J}$ is the inertia, $a$ and $b$ define the thrust dynamics, and
\begin{equation}
\setlength{\arraycolsep}{3pt}
\boldsymbol{\Omega}(\boldsymbol{\omega}) =
\begin{bmatrix}
    0          & -\omega_0 & -\omega_1 & -\omega_2 \\
    \omega_0   & 0         & \omega_2  & -\omega_1 \\
    \omega_1   & -\omega_2 & 0         & \omega_0 \\
    \omega_2   & -\omega_1 & -\omega_0 & 0
\end{bmatrix}. \nonumber
\end{equation}
The forces acting on the fixed-wing in the body-fixed frame can be written as
\begin{equation}
    \mathbf{f} = \sum_i\left(\mathbf{R}_{s_i}\mathbf{f}_{s_i}\right) + \delta_{th}\mathbf{R}_t\mathbf{e}_x - mg\mathbf{R}(\mathbf{q})^T\mathbf{e}_z +\mathbf{f}_d
\end{equation}
where $\mathbf{R}_{s_i}, \ \mathbf{R}_t \in SO(3)$ are the rotation of the aerodynamic surfaces and the thrust source, respectively, with respect to the body-fixed frame, $\mathbf{e}_x, \ \mathbf{e}_z$ are unit vectors in the $x$ and $z$ directions, respectively, $g$ is gravity, $\mathbf{f}_{s_i}$ is the force from each aerodynamic surface, and $\mathbf{f}_d$ is the drag force.

The forces from the $i^{th}$ aerodynamic surface are modeled using the flat plate model in \cite{flatplate}:
\begin{equation}
\mathbf{f}_{s_i}(\state_t, \inp_t) = \rho S_i||\mathbf{v}_{s_i}(\state_t, \inp_t)||\left(\mathbf{v}_{s_i}(\state_t, \inp_t) \cdot \mathbf{e}_z \right)\mathbf{e}_z
\end{equation}
where $\rho$ is the air density, $S_i$ is the surface area, and $\mathbf{v}_{s_i}$ is the surface velocity given as 
\begin{equation}
\mathbf{v}_{s_i} = \mathbf{R}_{s_i}^T(\mathbf{v_b}+\boldsymbol{\omega}\times\mathbf{r}_{h_i} + \gamma_i\mathbf{v}_{bw}) +(\mathbf{R}_{s_i}^T\boldsymbol{\omega}+\boldsymbol{\omega}_{s_i}) \times \mathbf{r}_{s_i}.
\end{equation}
Here, $\mathbf{v_b}=\mathbf{R}(\mathbf{q})^T\mathbf{v}$ is the linear velocity in the body-fixed frame, $\mathbf{r}_{h_i}$ is displacement from the center of mass to the stationary point on the surface with respect to the body frame origin, $\mathbf{r}_{s_i}$ is the displacement from the hinge point to the surface center of pressure, and $\boldsymbol{\omega}_{s_i}$ is the commanded surface rotation rate, and $\gamma_i$ is an empirically determined backwash coefficient. $\mathbf{v}_{bw}$ is the backwash velocity from the propeller, which is approximated using actuator disk theory
\begin{equation}
\mathbf{v}_{bw} = \left[\sqrt{\|\mathbf{v_p}\|^2+\frac{2\delta_{th}}{\rho S_{disk}}}-\|\mathbf{v}_p\|\right]\mathbf{e}_x
\end{equation}
where $\mathbf{v_p}$ is the freestream velocity at the propeller and $S_{disk}$ is the area of the actuator disk.
The drag force is modeled as 
\begin{equation}
\mathbf{f}_d = -\frac{1}{2}C_{b_d}\rho\|\mathbf{v}_b\|\mathbf{v}_b
\end{equation}
where $C_{b_d}$ is determined empirically.

The moment in the body-fixed frame, $\mathbf{m}$, is given as
\begin{align}
\mathbf{m} = \sum_i\left(\left(\mathbf{r}_{h_i}+\mathbf{R}_{s_i}\mathbf{r}_{s_i}\right) \times \mathbf{R}_{s_i}\mathbf{f}_{s_i} \right)
\end{align}



The stochastic, discrete-time dynamics applies a second-order Runge-Kutta (RK2) integration and Gaussian process noise, fit from hardware flight data (Fig. \ref{fig:histogram}).

\begin{figure}[t]
    \centering
    \includegraphics[trim={0 500 200 600},clip,width=1.0\columnwidth]{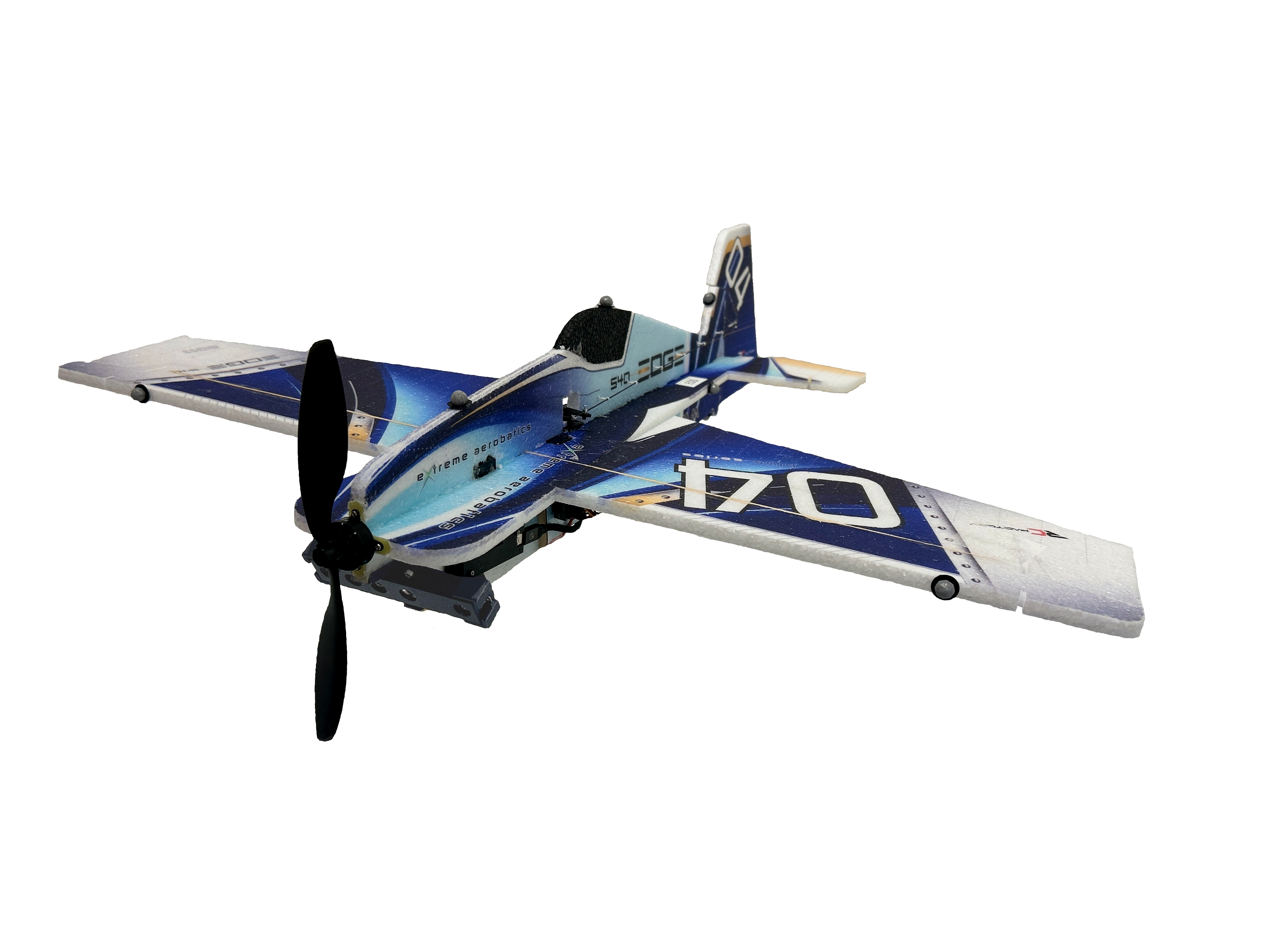}
    \caption{Edge540 aerial vehicle with RealSense D450 depth camera.}
    \label{fig:fixedwing}
\end{figure}

\begin{figure}[t]
    \centering
    \includegraphics[trim={0 0 0 0},clip,width=1.0\columnwidth]{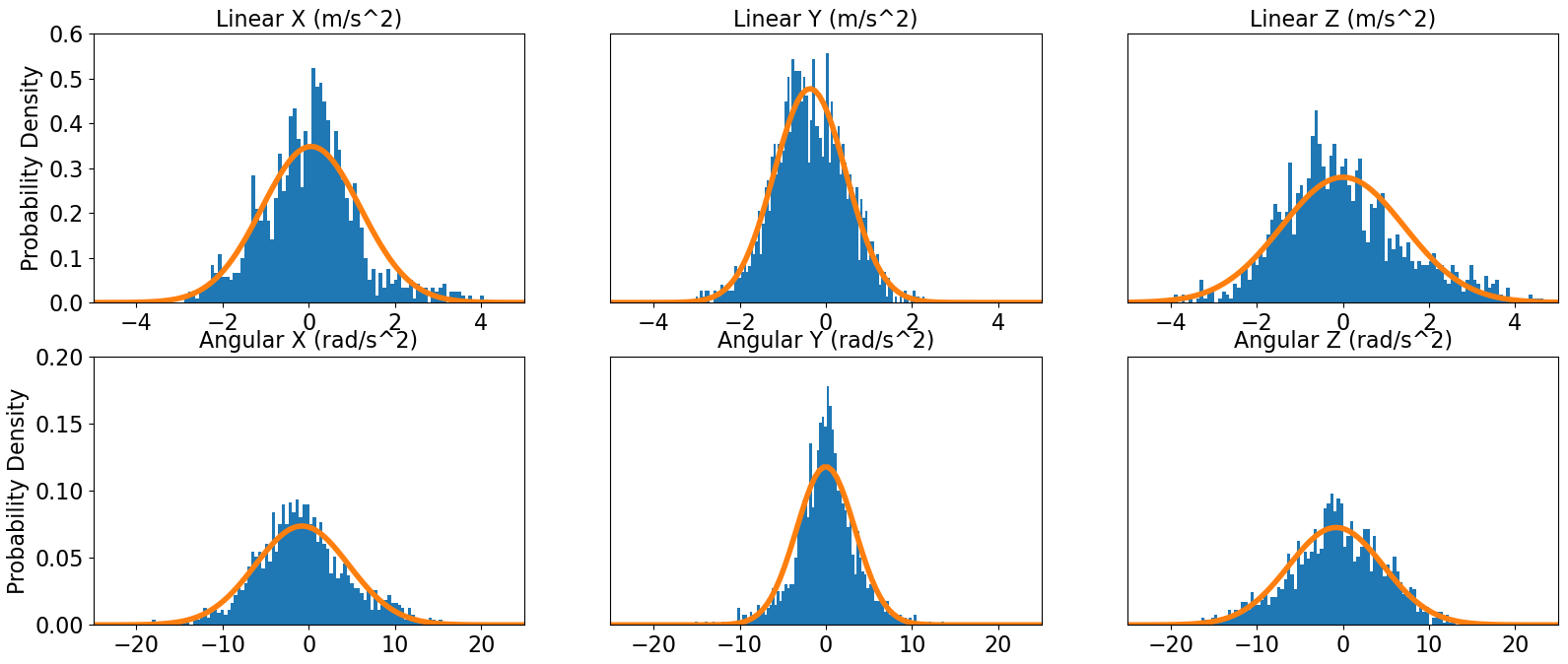}
    \caption{Histogram of fixed-wing aerial vehicle model acceleration errors. Fitted Gaussian distribution shown in orange.}
    \label{fig:histogram}
\end{figure}

The fixed-wing is equipped with a depth camera, which has an $87^\circ \times 58^\circ$ field of view. The sensor measurement, $\sensor_t$, is a flattened depth image.


\subsection{Actor Critic Training}

\label{subsection:depth-sensor-prediction}



\begin{figure*}[t]
    \centering
    \begin{subfigure}[b]{0.19\textwidth}
        \centering
        \includegraphics[width=\linewidth]{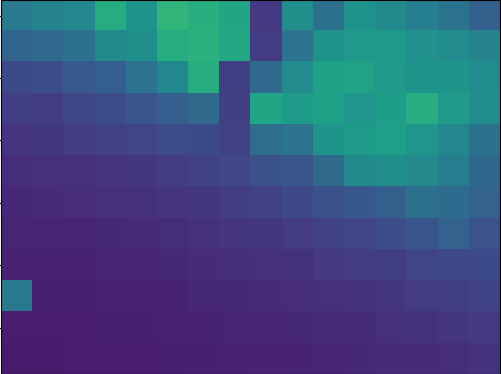}
        \caption{Current Input}
    \end{subfigure}\hfill
    \begin{subfigure}[b]{0.19\textwidth}
        \centering
        \includegraphics[width=\linewidth]{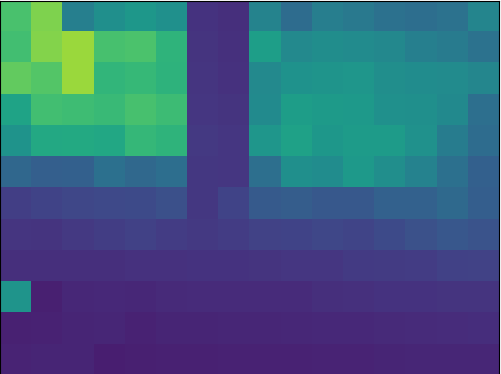}
        \caption{Future Target}
    \end{subfigure}\hfill
    \begin{subfigure}[b]{0.19\textwidth}
        \centering
        \includegraphics[width=\linewidth]{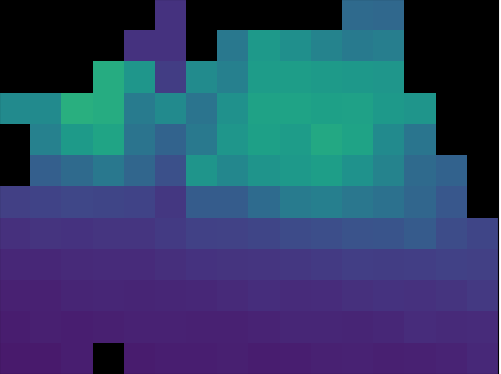}
        \caption{Geometric Prediction}
    \end{subfigure}\hfill
    \begin{subfigure}[b]{0.19\textwidth}
        \centering
        \includegraphics[width=\linewidth]{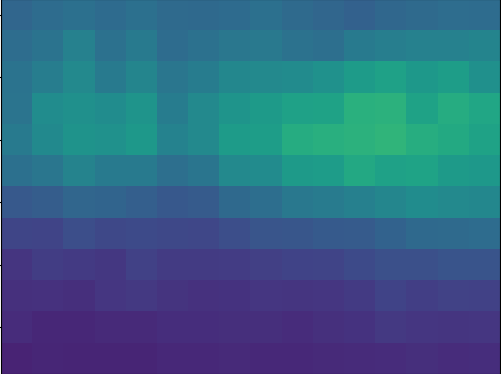}
        \caption{Fully Connected}
    \end{subfigure}\hfill
    \begin{subfigure}[b]{0.19\textwidth}
        \centering
        \includegraphics[width=\linewidth]{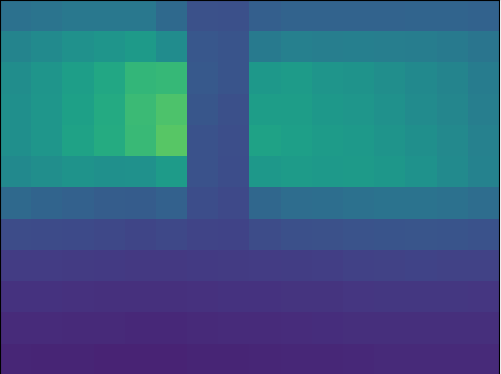}
        \caption{U-Net}
    \end{subfigure}

    \caption{Image prediction on decimated RealSense D450 depth image collected during fixed-wing flight. Depth is visualized as blue (near) to yellow (far).}
    \label{fig:image-prediction-models}
\end{figure*}

Similarly to Section \ref{subsection:rallycar-actor-critic}, the actor and critic were trained using the CleanRL \cite{cleanrl} implementation of TD3 \cite{td3}. The dynamics and sensor model previously described were used for training, and the stage cost was summed with a constraint violation penalty, $\gamma_c=1500$. The actor and critic, $\actor(\obs_t)$, $\critic(\obs_t, \inp_t)$ were parameterized identically as fully connected neural networks with two hidden layers, 256 neurons per layer, 10\% dropout, and ReLU activation functions.

The depth camera sensor measurement, $\sensor_t$, was limited to a resolution of $16 \times 12$ pixels. This low resolution enables rapid inference of the actor, critic, and sensor prediction models. Additionally, given the simplicity of these inputs compared to a high-resolution image, the learned models converge faster and generalize better. While this resolution may limit the ability of the system to detect thin obstacles, we found it sufficient for navigation in our targeted environments, where obstacles are $0.2$ to $0.4$ meters in radius.

To aid in efficient learning, the observations were preprocessed for the networks. The network inputs consisted of the z position, flattened rotation matrix, control surface angles, body-frame linear and angular velocities, range to goal, cosine and sine of the azimuth and elevation angles to the goal, and flattened depth image. These inputs were all normalized.

\subsection{Sensor Prediction}

As discussed in Section \ref{section:rlaugmentedpacnmpc}, estimates of future sensor measurements, $\sensorpredict_t$, must be obtained to warm-start PAC-NMPC with the actor policy and to evaluate the value function at the terminal states of sampled trajectories. We first evaluated image prediction through geometric projection of current measurements to future states, as was done for LiDAR measurement prediction for the rally car platform. 

First, the depth image was projected into 3D space, then assigned to pixels in the camera view at the future state. We evaluated both nearest pixel assignment and bilinear splatting. We found that this approach did not yield consistently accurate predictions and often left many pixels unassigned, making it unsuitable for evaluating the actor and critic at future states along sampled trajectories. Thus, we propose training a model to predict future state measurements with supervised learning.

We generated a training dataset entirely from simulation using the RL actor policy. The actor policy controlled the system through randomly sampled environments, yielding a sequence of states, ($\state_0$, $\state_1$, $\cdots$), and sensor measurements ($\sensor_0$, $\sensor_1$, $\cdots$). For each timestep $t$ in each of these environments, we added to the dataset the model inputs ($\state_t, \sensor_t,\state_{t+k}$) and targets ($\sensor_{t+k}$) for all $k \in [0, \cdots, \horizon]$. This produced a dataset that was well suited to predicting future sensor measurements along the entire trajectory horizon. 

We then trained and evaluated fully connected and U-Net convolutional \cite{unet} networks and found that the U-Net architecture yielded superior results (Table \ref{table:image-prediction-models}, Fig. \ref{fig:image-prediction-models}). These networks were trained on a smooth L1 loss plus an SSIM Loss \cite{ssim} and were evaluated on a 10\% held-out validation set.

\begin{table}[t]
\centering
\setlength{\tabcolsep}{3pt} 
\begin{tabular}{|c || c|| c| c| c| c| c|}
\hline
Model & FLOPs (M) & $\delta_1$ & $\delta_2$ & $\delta_3$ & REL & RMSE \\
\hline
\hline
Geometric (Nearest) & N/A & 0.133 & 0.217 & 0.306 & 2.838 & 0.634 \\
\hline
Geometric (Splat) & N/A & 0.186 & 0.309 & 0.419 & 2.240 & 0.548 \\
\hline
\hline
Fully Connected & 36.798 & 0.848 & 0.935 & 0.963 & 0.170 & 0.081 \\
\hline
U-Net & 42.283 & 0.923 & 0.965 & 0.983 & 0.087 & 0.058 \\
\hline
\end{tabular}
\vspace{2mm}
\caption{Performance of depth image prediction models on simulated data. $\delta_i$ is the percentage of pixels satisfying $\max(\frac{\hat d}{d},\frac{d}{\hat d})<1.25^i$ where $\hat d$ and $d$ are predicted and ground-truth depths. REL and RMSE are the mean absolute relative error and root mean squared error respectively.}
\label{table:image-prediction-models}
\end{table}

The fully connected network had two hidden layers, 4096 neurons per layer, SiLU activation functions, and 20\% dropout. This had a similar computational complexity, measured in FLOPs, as the U-Net architecture that we evaluated. 

The U-Net architecture had 3 encoder-decoder stages. The input depth image, $\sensor_t$, was normalized and passed through an initial convolutional layer which increases the number of image channels. We found 16 initial channels to produce adequate performance while allowing us to inference the network fast enough. Each stage consists of a convolutional layer, an RMS norm, a SiLU activation function, and 20\% dropout. The number of image channels is doubled and the resolution is halved with each stage. 

The initial and future states $\state_t$, $\state_{t+k}$ enter through the bottleneck of the network. First, the $SE(3)$ transform from the initial to future state, represented as a translation vector and rotation matrix, is computed and flattened. This transform is passed through a fully connected network with two hidden layers, 256 neurons, and SiLU activation functions. The resulting features are used to scale and shift the intermediate feature maps at the bottleneck of the network \cite{film}, thus conditioning the image prediction on the state transform. As is done with the actor and critic networks, we implement $\sensorpredictfunc$ as a stochastic model using MC dropout.

\subsection{Simulation Experiments}
The following experiments evaluated AC-PAC-NMPC on the fixed-wing aerial vehicle with a depth camera in challenging simulation environments. We compared our proposed approach against several baselines, performed an ablation study of the major novel algorithmic contributions, and evaluated our approach outside of the RL training distribution.

\subsubsection{Controller Configuration}

We used a quadratic stage cost
\begin{align}
\costfunc&(\state_t, \inp_t) = (\state_t-\goalstate)^T\mathbf{Q}(\state_t-\goalstate) + \inp_t^T\mathbf{R}\inp_t \nonumber \\
\mathbf{Q} &= 0.1 \cdot
\mathrm{diag}(
[1\ 1\ 1\ 0\ 0\ 0\ 0\ 1\ 1\ 1\ 0\ 0\ 0\ 1\ 1\ 1\ 0]) \nonumber\\
\mathbf{R} &= 0.01 \cdot \mathrm{diag}([1\ 1\ 1\ 1]).
\end{align}

The stage constraint, $\constraintfunc(\state_t,\inp_t,\sensor_{(t-\sensorhist): t})\le0$, enforced bounds on the speed, $\|\mathbf{v}\| < 7.0 \mathrm{\frac{m}{s}}$, the z position, $r_z \in [1.25, 3.25]$ and required the vehicle to maintain a minimum distance of $0.6 \mathrm{m}$ from all observed obstacle points from the latest depth images. Since the depth camera has a limited field of view, recently observed obstacles may leave the image as the aerial vehicle performs aerobatic maneuvers. To account for this, the  obstacle constraint utilizes a history of depth images. If the queried state lies outside of the latest depth image field of view, an image from $0.1$ sec prior is queried. This continues for a history of up to $\sensorhist =10$ prior depth images.

We optimized feedback policies over 11 timestep trajectories with $\Delta t=0.1$ sec at a replanning period of $H=0.1$ sec. In contrast to the rally car's configuration, we replan at every integration timestep. We found that this increased replanning rate allowed for improved reactivity to obscured obstacles given the increased vehicle speed and limited camera field of view. The optimization used $L=1$ prior policies, $M=1024$ trajectory samples per prior, and $\delta=0.05$. The trajectory costs were normalized before optimization to achieve tighter PAC bounds. The constraint violation bound penalty was set to $\gamma=2$ and the allowable threshold to $\epsilon_c = 0.1$.

\subsubsection{Environments}

\label{subsubsection:fixed-wing-environments}

We generated a distribution of environments consisting of randomly placed cylindrical obstacles in a rectangular room. The room width, length, and height were uniformly sampled between [10, 30], [20, 30], and [4, 5] meters, respectively. These obstacles were arranged in randomly placed semi-circle arrangements that create concave traps. These traps create challenging environments for the fixed-wing aerial vehicle, which is unable to reliably stop and turn around when caught in the traps. Consequently, successful navigation often requires long-range reasoning beyond the finite MPC horizon. The number of semi-circle traps was randomly sampled between 0 and 6, and were oriented towards the initial position of the plane, $\pm \frac{\pi}{2}$. The obstacle radii were sampled between 0.2 and 0.4 meters. Obstacles were not allowed to overlap. The simulated depth camera included observations of the cylindrical obstacles, walls, ceiling, and floor. An example environment is shown in Figure \ref{fig:fixedwing-sim-environment}.

\begin{figure}[t]
    \centering
    \includegraphics[trim={0 0 0 0},clip,width=1.0\columnwidth]{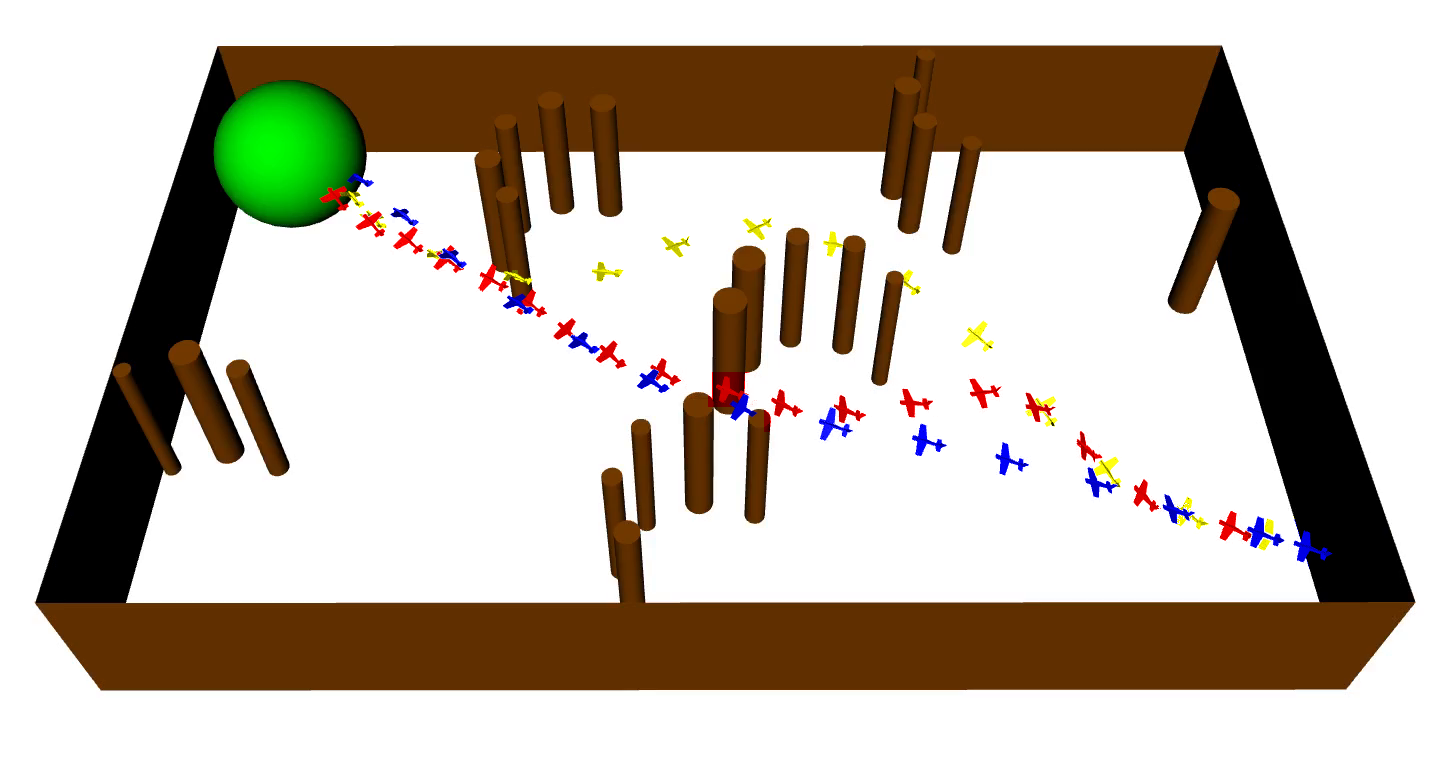}
    \caption{Simulation environment example. Ceiling/floor not shown. RL actor in yellow, PAC-NMPC with global planner in red, AC-PAC-NMPC in blue.}
    \label{fig:fixedwing-sim-environment}
\end{figure}

\subsubsection{Baseline Comparison}

We compared our approach against several baselines in these simulated environments. First, we compared against the RL actor policy itself, which was trained on this distribution of environments. Then, we compare against a sequence of increasingly complex variations of PAC-NMPC, each introducing an additional capability. These baselines isolate and compare the contributions of mapping, global planning, and unknown space penalties against uncertainty-aware RL augmentation. 

First, we evaluated PAC-NMPC using a history of depth images for its constraint and a quadratic terminal cost:
\begin{align}
\costfunc_f&(\state_t, \inp_t) =(\state_t-\goalstate)^T\mathbf{Q}(\state_t-\goalstate) \nonumber \\
&\hphantom{(\state_t, \inp_t)}+ (\boldsymbol{\eta}_t-\boldsymbol{\eta}_G)^T\mathbf{Q}_\eta(\boldsymbol{\eta}_t-\boldsymbol{\eta}_G) \nonumber \\
&\hphantom{(\state_t, \inp_t)}+ (\mathbf{v}_t-\mathbf{v}_G)^T\mathbf{Q}_v(\mathbf{v}_t-\mathbf{v}_G), \nonumber \\
\mathbf{d}_t &= \mathbf{r}_G - \mathbf{r}_t, \quad
\boldsymbol{\eta}_G = \mathrm{eul}\left(\mathbf{d}_t\right), \quad
\mathbf{v}_G = v_G\cdot\mathbf{d}_t/\|\mathbf{d}_t\|, \nonumber \\
\mathbf{Q} &=\mathrm{diag}([1\ 1\ 1\ 0\ 0\ 0\ 0\ 1\ 1\ 1\ 0\ 0\ 0\ 0.1\ 0.1\ 0.1\ 0]), \nonumber \\
\mathbf{Q}_\eta &= 10 \cdot \mathrm{diag}([1\ 1\ 1]), \quad \mathbf{Q}_v = 0.1 \cdot \mathrm{diag}([1\ 1\ 1]).
\end{align}
Here, $\boldsymbol{\eta}_t$, $\mathbf{v}_t$ are the Euler angles and linear world velocity of the aerial vehicle at state $\state_t$. The desired Euler angles, $\boldsymbol{\eta}_G$, and velocity, $\mathbf{v}_G$, are constructed from the vector between the robot and the goal, $\mathbf{d}_t$, where the target speed is set to $v_G = 6\mathrm{\frac{m}{s}}$.

To improve awareness of previously observed obstacles, the next baseline maintains a map of the environment. We maintained an occupancy voxel grid, which was constructed and maintained using the the widely adopted Nav2 framework \cite{nav2}. The voxel grid had a resolution of $0.1\mathrm{m} \times0.1\mathrm{m}\times0.1\mathrm{m}$ and was used in place of the depth image history for the obstacle constraint function.

To improve the robot's ability to avoid local traps in the environment, we incorporated a global planner into the next baseline. We implemented a 3-dimensional A* global planner. Each planning iteration, the A* search algorithm was run on the voxel grid to find the shortest path to the goal. These paths were pruned, smoothed, and parameterized by velocity, as was done in \cite{agile-fixed-wing}. Target velocity along the path is decreased as the path curvature increases: straight segments were assigned $6\mathrm{\frac{m}{s}}$ and segments with the maximum curvature were assigned $3\mathrm{\frac{m}{s}}$. This encourages the aerial vehicle to slow down, resulting in a tighter turning radius, as the curvature of the path increases. A receding horizon goal, $\state_{A^*}$ was selected at a time horizon of $1$ sec along the path to form a quadratic terminal cost,
\begin{align}
\costfunc_f&(\state_t, \inp_t) =(\state_t-\state_{A^*})^T\mathbf{Q}(\state_t-\state_{A^*}) \nonumber \\
&\hphantom{(\state_t, \inp_t)}+(\boldsymbol{\eta}_t-\boldsymbol{\eta}_{A^*})^T\mathbf{Q}_\eta(\boldsymbol{\eta}_t-\boldsymbol{\eta}_{A^*}) \nonumber \\
&\hphantom{(\state_t, \inp_t)}+ (\mathbf{v}_t-\mathbf{v}_{A^*})^T\mathbf{Q}_v(\mathbf{v}_t-\mathbf{v}_{A^*}), \nonumber \\
\mathbf{Q} &=
\mathrm{diag}(
[10\ 10\ 10\ 0\ 0\ 0\ 0\ 1\ 1\ 1\ 0\ 0\ 0\ 0.1\ 0.1\ 0.1\ 0]), \nonumber \\
\mathbf{Q}_\eta &= 10 \cdot \mathrm{diag}([1\ 1\ 1]), \quad \mathbf{Q}_v = \mathrm{diag}([1\ 1\ 1]),
\end{align}
where $\boldsymbol{\eta}_{A^*}$ and $\mathbf{v}_{A^*}$ are the Euler angles and linear world velocity of the receding horizon goal as parameterized by the smoothed A* global path.

\begin{table}[t]
\centering
\begin{tabular}{|c | c || c| c|}
\hline
\multicolumn{2}{|c||}{Approach} & Success & Cost \\
\hline
\hline
\multicolumn{2}{|c||}{RL Actor Network} & 81\% & \textbf{468.1} \\ 
\hline
\hline
\multirow{5}{*}{PAC-NMPC} & Depth Image History & 52\% & 757.7 \\
\cline{2-4}
 & Map & 63\% & 810.2 \\
\cline{2-4}
 & Map \& A* & 74\% & 738.8 \\
\cline{2-4}
 & Map \& A* Unknown & 85\% & 718.14 \\
\cline{2-4}
 & Actor-Critic (ours) & \textbf{90\%} & 513.3 \\
\hline
\end{tabular}
\vspace{2mm}
\caption{Baseline Comparison Results}
\label{table:fixedwing-sim-baselines}
\end{table}

To improve the robot's ability to avoid obscured obstacles in the environment, the last baseline applies an unknown-space penalty to the A* planner. Each cell in the voxel grid is assigned as either occupied, unoccupied, or unobserved based on ray casting from the depth camera poses. This baseline assigned an extra cost of $100$ to the A* planner for traversing through unobserved cells.

Finally, our approach utilizes AC-PAC-NMPC without mapping or a global planner. We ran our experiments over 100 environments sampled from the training distribution. Trivial environments in which no obstacles blocked the shortest path to the goal were discarded.

We found that our approach had the highest success rate, 90\%, of reaching the goal without violating the obstacle constraints (Table \ref{table:fixedwing-sim-baselines}). Additionally, we found that our approach accomplished this with the lowest average accumulated cost over successful trials out of all the PAC-NMPC baselines. This indicates that in these environments, the RL guidance was most effective in enabling successful, low cost navigation, even outperforming PAC-NMPC with mapping and a global planner. Additionally, since our approach had a higher success rate than the RL actor alone, it was successful in improving the safety of the RL policy while benefiting from its improved long range behavior. 

\subsubsection{Ablation Study}

We next performed an ablation study to evaluate the impact of two major contributions of this work: RL-based warm-start and learned sensor prediction. This study was run in the same sampled environments of the prior section. To evaluate RL-based warm-start, we replaced it with the default PAC-NMPC warm-start approach as presented in \cite{pacnmpc}. To evaluate the learned sensor prediction, we used geometric projection with nearest pixel assignment. We found that the lack of either of these components resulted in successful navigation in only a small percentage of trials (Table \ref{table:fixedwing-sim-ablation}).

\begin{table}[t]
\centering
\begin{tabular}{|c | c || c | c|}
\hline
RL-based Warm-start & Learned Prediction & Success & Cost \\ 
\hline
\hline
& \checkmark & 3\% & 1001.2 \\
\hline
\checkmark &  & 11\% & 889.7 \\
\hline
\checkmark & \checkmark & \textbf{90\%} & \textbf{513.3} \\
\hline
\end{tabular}
\vspace{2mm}
\caption{Ablation Study Results}
\label{table:fixedwing-sim-ablation}
\end{table}

\begin{figure}[t]
    \centering
    \includegraphics[trim={0 0 0 0},clip,width=1.0\columnwidth]{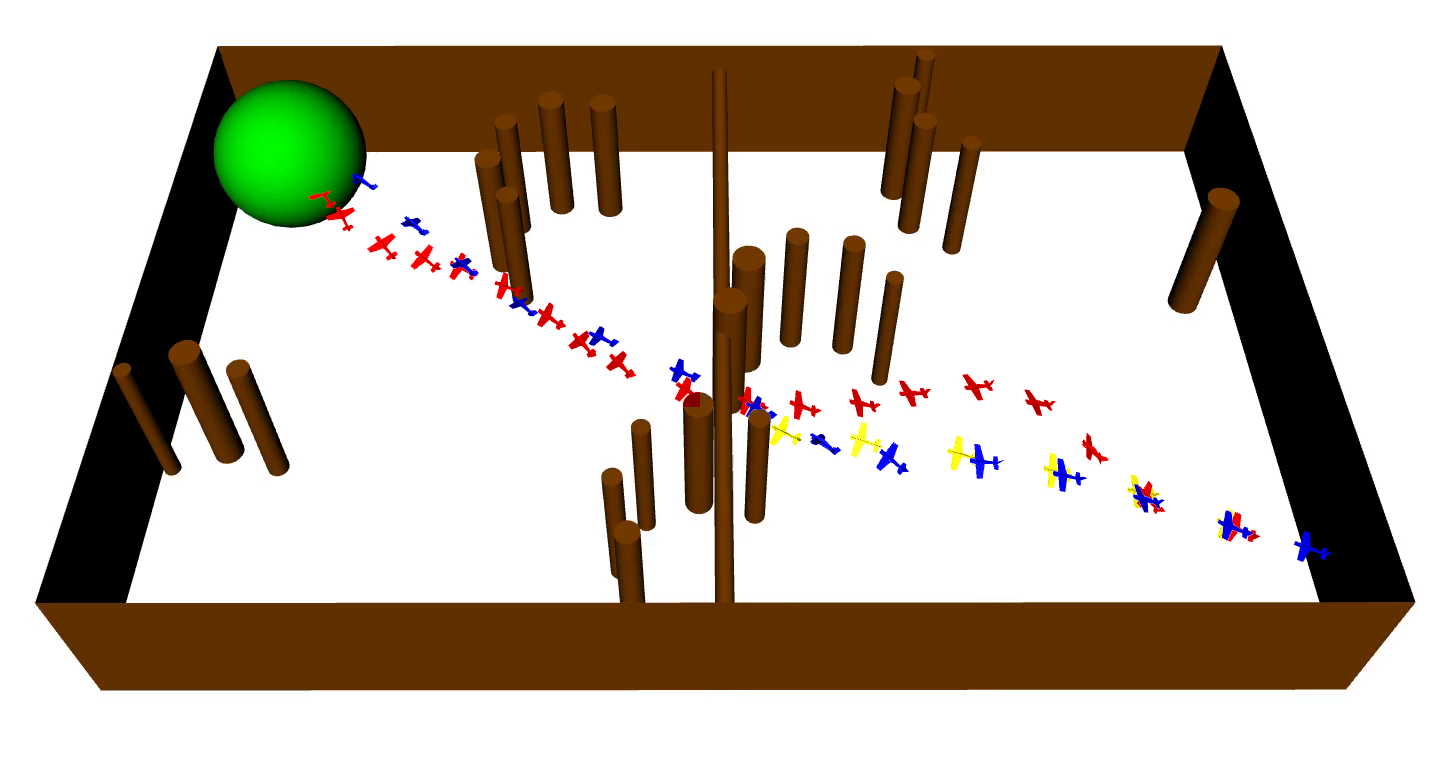}
    \caption{Out-of-distribution simulation environment example. Ceiling/floor not shown. RL actor (collision) in yellow, PAC-NMPC with global planner in red, AC-PAC-NMPC in blue.}
    \label{fig:fixedwing-ood-sim-environment}
\end{figure}

We believe that poor performance when lacking learned sensor prediction is caused by the inaccurate and unsuitable predictions generated by geometric projection, as demonstrated in Section \ref{subsection:depth-sensor-prediction}. The RL-based warm-start was also particularly important for this system. One possible explanation is that without it, the critic network may be producing poor approximations of the value function. Because the critic is primarily trained on observation-input pairs generated by the actor policy, its estimates may be unreliable along trajectories that substantially deviate from the actor. This was not observed on the rally car with LiDAR, perhaps because its dynamics and sensor models have significantly lower dimensionality, thus reducing the impact of this effect. These results demonstrate the importance of these novel contributions when utilizing AC-PAC-NMPC to control more complex systems.

\subsubsection{Out of Distribution Experiment}

We evaluated the performance of our approach in environments that are outside of the RL training distribution. Since the actor and critic models were only trained in the presence of vertical cylindrical obstacles, we altered the testing environments by adding a horizontal obstacle. This is a relatively simple change to the testing environments for MPC, but may represent a significant distribution shift for the RL actor and critic models.

We reused the same testing environments from the previous subsections, but added a horizontal obstacle in the middle of each room. The horizontal obstacle was randomly sampled between heights of $1.25\mathrm{m}$ and $3.25\mathrm{m}$. Obstacles sampled within $0.6\mathrm{m}$ of the height of the goal were resampled since this blocks a large portion of the feasible flight space. The horizontal obstacle had a radius of $0.2\mathrm{m}$. An example environment is shown in Figure \ref{fig:fixedwing-ood-sim-environment}.

\begin{table}[t]
\centering
\begin{tabular}{|c | c || c| c|}
\hline
\multicolumn{2}{|c||}{Approach} & Success & Cost \\
\hline
\hline
 \multicolumn{2}{|c||}{RL Actor Network} & 46\% & \textbf{481.8} \\ 
\hline
\hline
\multirow{2}{*}{PAC-NMPC} & Map \& A* Unknown & \textbf{76\%} & 732.6 \\
\cline{2-4}
 & Actor-Critic (ours) & \textbf{76\%} & 512.1 \\
\hline
\end{tabular}
\vspace{2mm}
\caption{Out Of Distribution Experiment Results}
\label{table:fixedwing-ood-results}
\end{table}
In the presence of these out-of-distribution obstacles, the RL actor policy violated obstacle constraints in most trials, dropping to a 46\% success rate. Our approach, on the other hand, succeeded in 76\% of trials, demonstrating improved ability to transfer to environments outside of the training distribution. In these environments, our approach matched the success rate of the best baseline, but achieved an improved average accumulated cost over successful trials (Table \ref{table:fixedwing-ood-results}). These results suggest that our approach is able to incorporate some of the improved long range behavior of the RL policy while being more robust to shifts from the training distribution.
\subsection{Hardware Experiments}

\begin{figure*}[t]
    \centering
    \begin{subfigure}[b]{0.49\textwidth}
        \centering
        \includegraphics[
            trim={100 200 0 0},
            clip,
            width=\linewidth
        ]{images/timelapse1.png}
    \end{subfigure}\hfill
    \begin{subfigure}[b]{0.49\textwidth}
        \centering
        \includegraphics[
            trim={0 150 100 50},
            clip,
            width=\linewidth
        ]{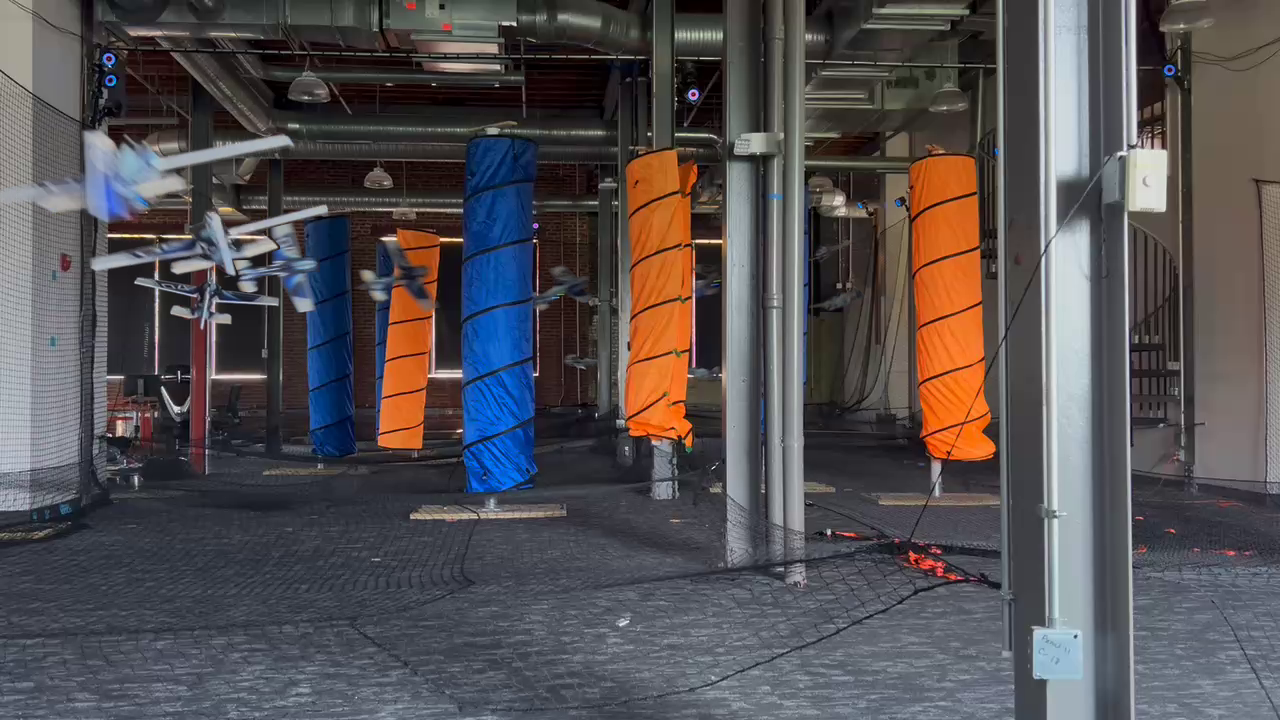}
    \end{subfigure}

    \caption{Example hardware environments with time-lapse trajectories of the aerial vehicle controlled by our approach.}
    \label{fig:fixedwing-hardware-environments}
\end{figure*}

We evaluated our approach on a Twisted Hobbys 32-inch wing-span fixed-wing aerial vehicle equipped with an Intel RealSense D450 depth camera module and Intel Vision Processor D4 Board. An on-board Arduino Uno Q running a ROS2 docker image streamed the depth images to the controller laptop over WiFi at 50Hz using the realsense-ros wrapper.

The experiments took place in an Optitrack motion capture facility, which provided pose and velocity state estimates to the controller laptop at 180Hz over ethernet. The laptop had an Intel i9-13900H CPU and a Nvidia GeForce RTX 4080 Laptop GPU. 
The laptop executed the control algorithms and transmitted resulting control inputs to the aerial vehicle using a Futaba T6K transmitter.

Environment generation followed the same procedure as Section $\ref{subsubsection:fixed-wing-environments}$, with a fixed room size and 1 to 2 sampled semi-circle traps. I-beams and a spiral staircase were permanent obstacles across all environments. Play tunnels supported by PVC pipes were placed at the sampled obstacle positions. Example hardware environments are shown in Figure \ref{fig:fixedwing-hardware-environments}. We evaluated our approach in 10 randomly sampled environments.

\begin{figure}[t]
    \centering
    \includegraphics[trim={0 0 0 0},clip,width=1.0\columnwidth]{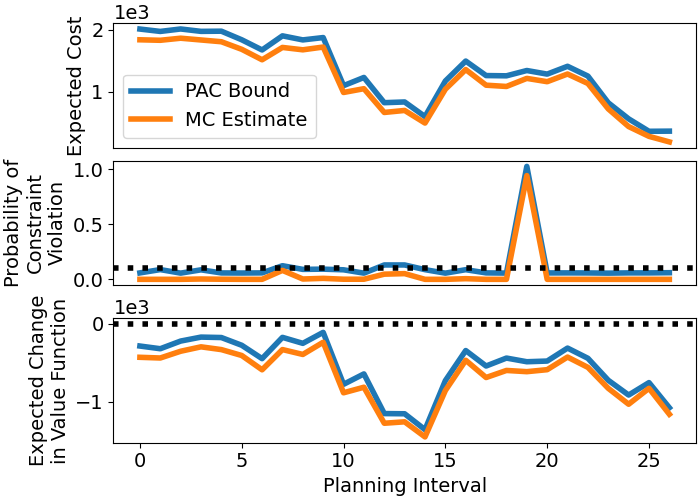}
    \caption{Optimized PAC bounds compared to Monte Carlo estimates at each planning interval for one trial.}
    \label{fig:fixedwing-bounds}
\end{figure}

\begin{figure}[t]
    \centering
    \includegraphics[trim={0 13 0 64},clip,width=0.49\columnwidth]{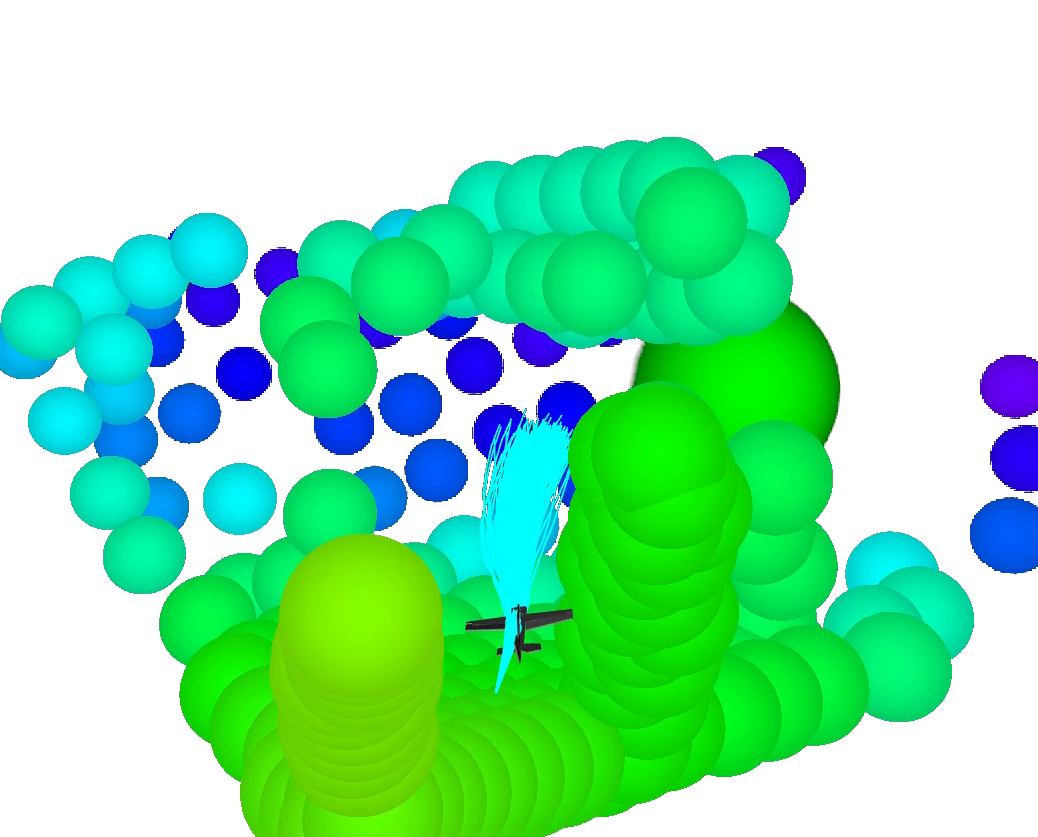}%
    \hfill
    \includegraphics[trim={0 13 0 64},clip,width=0.49\columnwidth]{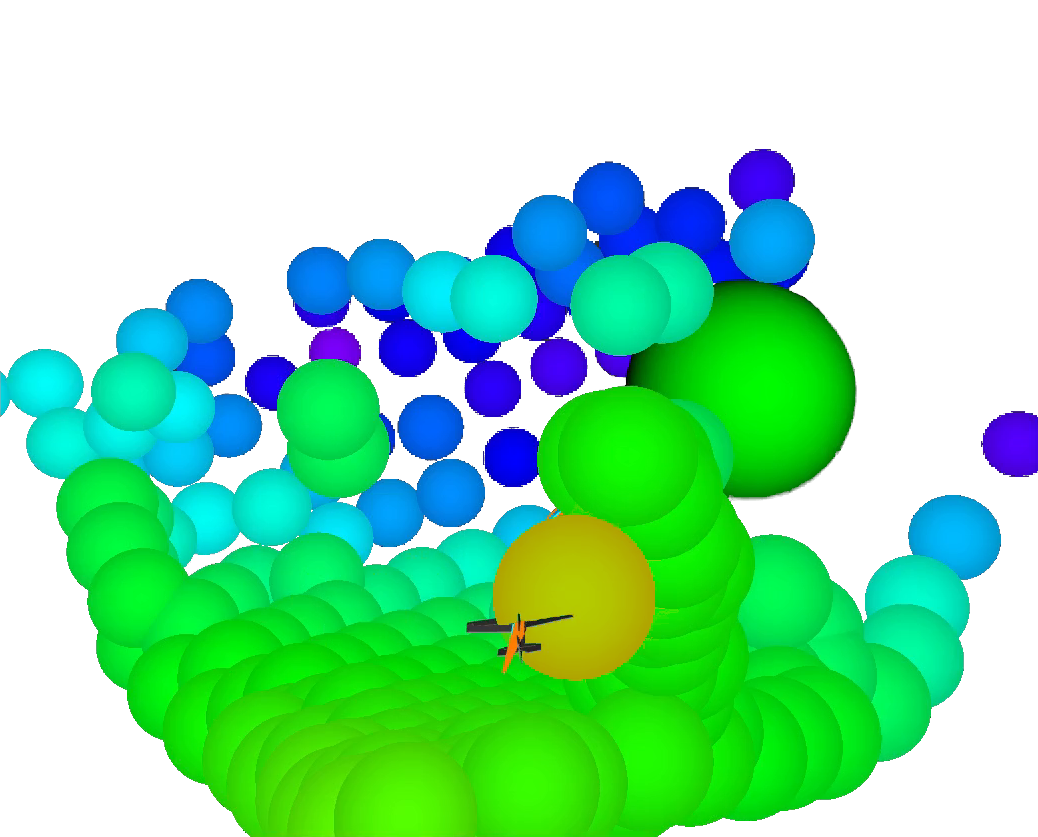}
    \caption{Spurious noise in consecutive depth camera measurements, which caused a spike in $\mathcal{C}^+_{\alpha}$.}
    \label{fig:sensor-noise}
\end{figure}

We compared our approach against the RL actor and PAC-NMPC with mapping, an A* global planner, and unknown space costs. The actor, critic, and sensor prediction models were trained only in simulation and transferred to hardware without any additional fine-tuning. The depth images were decimated to $80 \times 60$ pixels on the Intel Vision Processor D4 Board. These images contained a lot of noise between foreground obstacles and the background. To reduce this noise, we ran a gradient filter on the Arduino Uno Q that invalidated pixels at which the local depth gradients exceeded 10\% of the measured depth. When running PAC-NMPC with mapping, the images were further decimated to $40 \times 30$ pixels with a median filter to remove outliers. When running the RL actor and AC-PAC-NMPC, images were decimated to $16 \times 12$ pixels (input resolution to the learned models) with the median filter.

We found that our approach had the highest success rate, 80\%, and achieved the lowest average accumulated cost over successful trials (Table \ref{table:fixedwing-hardware-results}). Figure \ref{fig:fixedwing-bounds} displays the optimized PAC bounds at each planning interval for one of the environments. In this trial, there was a spike in the probability of constraint violation bound, $\mathcal{C}^+_{\alpha}$, which was caused by spurious noise in the depth camera data (Fig. \ref{fig:sensor-noise}).

Both baselines only achieved a 40\% success rate. The RL actor demonstrated a significantly lower success rate on hardware than in simulation. Two possible contributions are the mismatch of the modeled dynamics compared to the physical system and the addition of sensor noise. This suggests that our approach is better able to bridge the sim-to-real gap than RL policies alone. 

The PAC-NMPC baseline also demonstrated significantly degraded hardware performance. During testing, we observed that noise in the depth camera still occasionally progressed to the mapping stage, despite the gradient filter and median decimation. Although the Nav 2 mapping algorithm attempts to remove false observations through ray casting, we observed that some intermittent noise remained in the map, which might have adversely affected planning performance. This may indicate that our approach, which limits the integration of sensor noise by only querying a history of depth images, could be more robust to this form of sensor noise. 

\begin{table}[t]
\centering
\begin{tabular}{|c | c || c| c|}
\hline
\multicolumn{2}{|c||}{Approach} & Success & Cost \\
\hline
\hline
\multicolumn{2}{|c||}{RL Actor Network} & 40\% & 403.8 \\ 
\hline
\hline
\multirow{2}{*}{PAC-NMPC} & Map \& A* Unknown & 40\% & 325.4 \\
\cline{2-4}
 & Actor-Critic (ours) & \textbf{80\%} & \textbf{252.2} \\
\hline
\end{tabular}
\vspace{2mm}
\caption{Fixed-wing aerial vehicle hardware results.}
\label{table:fixedwing-hardware-results}
\end{table}

\begin{figure}[t]
    \centering
    \includegraphics[trim={50 135 75 175},clip,width=1.0\columnwidth]{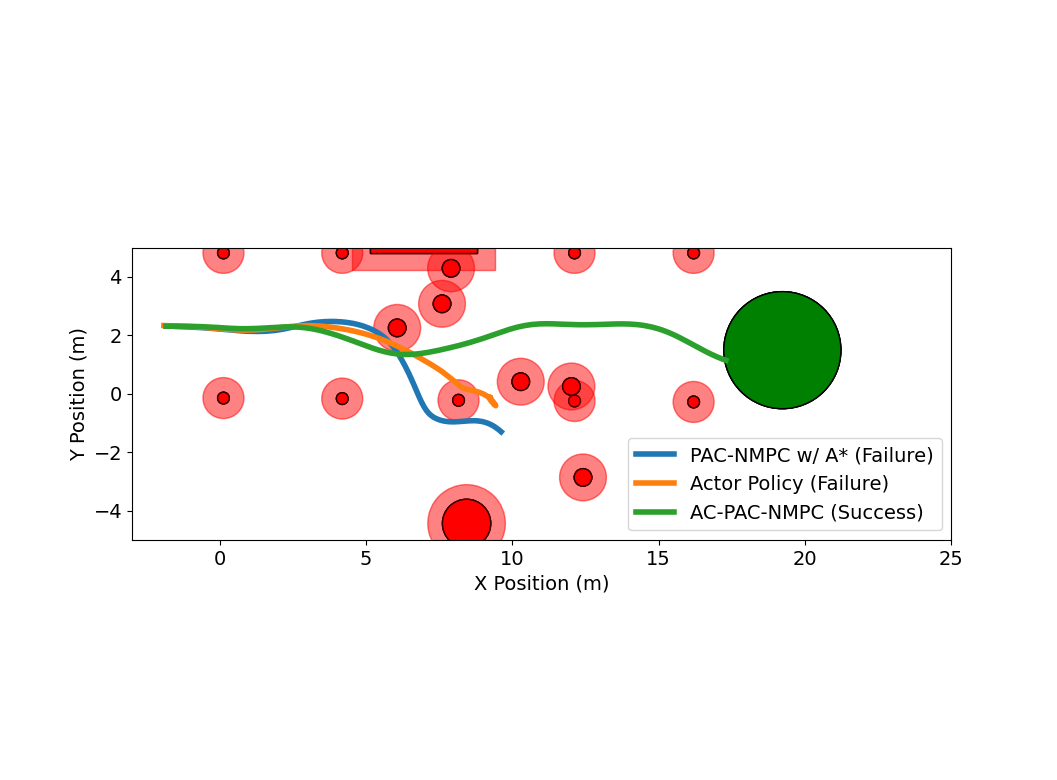}
    \caption{Flight paths from one hardware trial.}
    \label{fig:trial1-paths}
\end{figure}

Flight paths from one of the trials are shown in Figure \ref{fig:trial1-paths}. These results demonstrate that our approach successfully controlled, in real-time, a high-dimensional, nonlinear, robotic system through an unknown environment using only local perception while providing probabilistic guarantees of safety.


\section{Analysis}

In this section, through analysis of the value function and optimal policy in the presence of constraints, we provide theoretical justification for using actor and critic models within NMPC frameworks to explicitly enforce constraints.

Often, when training actor and critic models in the presence of constraints, the RL training cost, $\rlcostfunc(\state_t,\inp_t)$, is defined as
\begin{align}
\rlcostfunc(\state_t,\inp_t)= \costfunc(\state_t,\inp_t) + \gamma_c \constraintfunc(\state_t,\inp_t)
\end{align}
where $\gamma_c > 0$ is a constraint violation penalty. For the purpose of this analysis, we will assume the training constraint is strictly non-negative, $\constraintfunc(\state_t, \inp_t) \geq 0$, where $\constraintfunc(\state_t, \inp_t) = 0$ indicates constraint satisfaction. The constraint violation penalty is either heuristically selected and static, or it can be interpreted as a Lagrange multiplier and learned during training, as is common in safe-RL approaches.

However, there is no guarantee that the resulting optimal policy will be constraint-free. Furthermore, there is no guarantee that a learned approximate policy will be constraint-free, especially outside of the training distribution. Through the following analysis, we show that, given a sufficiently large constraint penalty and the existence of a feasible policy, then for every policy generated by the optimal policy, there exists nearby control inputs that strictly satisfy the constraints while still descending the value function.

\subsection{One-Step Feasible Descent}

First, we prove that for a large enough $\gamma_c$, there exists some input in the neighborhood of the optimal policy that satisfies the constraints while descending the value function. We first show there exist non-optimal inputs in a local region around the optimal policy that still descend the value function.
\begin{lemma}
    For a given $\state_t$, let the value function $V:\mathcal{X} \mapsto \mathbb{R}$ be continuous and bounded, $p(\state_{t+1} \mid \state_t, \inp_t)$ be weakly continuous in $\inp_t$, and $g(\state_t,\inp_t) \succ 0$. Then, there exists a local region $B_{\delta}(\inp^*_t)=\{\inp_t\in\mathbb{R}^n:||\inp_t-\inp^*_t||<\delta\}$, such that $\EV-V(\state_t)<0$.
    \label{lemma:region}
\end{lemma}

\begin{proof}
    Given the optimal input $\inp^*_t$, we can write 
    \begin{equation}
        V(\state_t) = g(\state_t,\inp^*_t)+\beta \mathbb{E}\left[V(\state_{t+1}) \mid \state_{t}, \inp_{t}^* \right]
        \label{eq:sto_opt_vf}
    \end{equation}
    where $\beta\in(0,1)$ and for some non-optimal input $\inp_t$, we have
    \begin{equation}
        V(\state_t) \le \rlcostfunc(\state_t, \inp_t) + \beta \EV.
    \end{equation}
    We define expected one-step change in the value function as
    \begin{align}
        \Delta V (\state_t,\inp_t) &= \mathbb{E}\left[V(\state_{t+1}) \mid \state_{t}, \inp_{t} \right]-V(\state_t) \\ \nonumber & \ge (1-\beta)\mathbb{E}\left[V(\state_{t+1}) \mid \state_{t}, \inp_{t} \right] - \rlcostfunc(\state_t, \inp_t)
        \\ \nonumber & \ge  - \rlcostfunc_{\beta}(\state_t, \inp_t).
        \label{eg:sto_deltav}
    \end{align}
    We note that $\Delta V(\state_t,\inp^*_t) = -g_{\beta}(\state_t,\inp^*_t)$ and $g_{\beta}(\state_t,\inp_t)\succ 0$ for sufficiently large $\beta$ and $\rlcostfunc(\state_t, \inp_t)$. Given $V(\state_t)$ is continuous and bounded and $p(\state_{t+1} \mid \state_t, \inp_t)$ is weakly continuous in $\inp_t$, $\EV$ is continuous, and thus, $\Delta V (\state_t,\inp_t)$ is continuous. Therefore, $\exists \delta > 0$ such that $||\inp_t-\inp^*_t||< \delta$ where $\Delta V(\state_t,\inp_t) < 0$. \qedhere
\end{proof}

Next, under some assumptions, we show that as $\gamma_c$ increases, expected constraint-to-go approaches lower bound $\epsilon$.

\begin{lemma}
    Let $\rlcostfunc(\state_t, \inp_t)= \costfunc(\state_t,\inp_t) + \gamma_c \constraintfunc(\state_t,\inp_t)$, where $\gamma_c>0$, $\costfunc(\state_t,\inp_t) \succ 0$, and $\constraintfunc(\state_t,\inp_t) \succeq 0$. For a given $\state_t$, suppose that $\Pi_\epsilon^c(\state_t) = \{\bpi | \bpi(\state_t) \in \mathcal{U}^c(\state_t), V^\constraintfunc_{\bpi}(\state_t) < \epsilon\} \neq \varnothing$ where $\mathcal{U}^c(\state_t) = \left\{\inp_t|\constraintfunc(\state_t, \inp_t)=0\right\}$, and $V^\constraintfunc_{\bpi}(\state_t) = \mathbb{E} \left[\sum_{i=t}^\infty \beta^{i-t}\constraintfunc(\state_i,\bpi(\state_i)) \right]$. Then there exists a large enough $\gamma_c$ such that $V^\constraintfunc_{\bpi_{\gamma_c}^*}(\state_t) < \epsilon$, where $\bpi_{\gamma_c}^*$ is the optimal policy for a corresponding $\gamma_c$.
    \label{lemma:limit}
\end{lemma}

\begin{proof}
    We can write 
    \begin{align}
        V_{\bpi_{\gamma_c}^*}(\state_t) &= V^\costfunc_{\bpi_{\gamma_c}^*}(\state_t) + \gamma_c V^\constraintfunc_{\bpi_{\gamma_c}^*}(\state_t) \\
        V^\costfunc_{\bpi_{\gamma_c}^*}(\state_t) &= \mathbb{E} \left[\sum_{i=t}^\infty \beta^{i-t}\costfunc(\state_i,\bpi_{\gamma_c}^*(\state_i)) \right] \\
        V^\constraintfunc_{\bpi_{\gamma_c}^*}(\state_t) &= \mathbb{E} \left[\sum_{i=t}^\infty \beta^{i-t}\constraintfunc(\state_i,\bpi_{\gamma_c}^*(\state_i)) \right]
    \end{align}
    Since $\bpi_{\gamma_c}^*$ is optimal, $\forall \bpi \in \Pi_\epsilon^c(\state_t)$
    \begin{align} 
        V^\costfunc_{\bpi_{\gamma_c}^*}(\state_t) + \gamma_c V^\constraintfunc_{\bpi_{\gamma_c}^*}(\state_t) \leq V^\costfunc_{\bpi}(\state_t) + \gamma_c V^\constraintfunc_{\bpi}(\state_t).
    \end{align} 
    Since $V^\costfunc_{\bpi_{\gamma_c}^*}(\state_t)\geq0$ 
    \begin{align} 
        V^\constraintfunc_{\bpi_{\gamma_c}^*}(\state_t) &\leq \frac{1}{\gamma_c}V^\costfunc_{\bpi}(\state_t) + V^\constraintfunc_{\bpi}(\state_t) \\
        V^\constraintfunc_{\bpi_{\gamma_c}^*}(\state_t) &< \frac{1}{\gamma_c}V^\costfunc_{\bpi}(\state_t) + \epsilon
    \end{align} 
    Therefore, for a large enough $\gamma_c < \infty$, $V^\constraintfunc_{\bpi_{\gamma_c}^*}(\state_t) < \epsilon$. \qedhere
\end{proof}

Given the previous lemmas and a bound on $\epsilon$, we show that there exists some feasible input in the neighborhood of the optimal policy that descends the value function.

\begin{theorem}
    Let $\delta_{min} = \min_{\bpi \in \Pi_\epsilon^c} \left[\delta_{\bpi}\right] > 0$ where $\delta_{\bpi}$ is the neighborhood from Lemma \ref{lemma:region} for $V_{\bpi}$. For the compact subset $\mathcal{U}_{\delta_{min}} = \{ \inp_t \in \mathcal{U} \mid \mathrm{dist}(\inp_t,\mathcal{U}^c) \geq \delta_{min} \}$, by the Extreme Value Theorem, $\constraintfunc_{\delta_{min}} = \min_{\inp_t \in \mathcal{U}_{\delta_{min}}} [\constraintfunc(\state_t,\inp_t)] > 0$.

    For a given $\state_t$, and the assumptions of Lemmas \ref{lemma:region} and \ref{lemma:limit}, if $\Pi_\epsilon^c \neq \varnothing$ with $\epsilon < \constraintfunc_{\delta_{min}}$, then there exists a feasible input, $\inp_t^c \in \mathcal{U}^c(\state_t)$, in the neighborhood of $\inp^*_t = \bpi^*_{\gamma_c}(\state_t)$ which satisfies $\Delta V(\state_t, \inp_t^c)< 0$.
    \label{theorem:onestep}
\end{theorem}

\begin{proof}

By Lemma \ref{lemma:limit}, there exists $\gamma_c$ large enough such that $V^\constraintfunc_{\bpi^*_{\gamma_c}}(\state_t) < \epsilon  < \constraintfunc_{\delta_{min}}$. Suppose that for every k in the sequence $\gamma_c^k \rightarrow \infty$,
\begin{equation}
\mathrm{dist}(\bpi_{\gamma_c^k}^*(\state_t), \mathcal{U}^c(\state_t)) \geq \delta_{min} \ \forall k.
\end{equation}
Therefore,
\begin{equation}
V^\constraintfunc_{\bpi_{\gamma_c^k}^*} \geq \constraintfunc(\state_t, \bpi^*_{\gamma_c^k}(\state_t)) \geq \constraintfunc_{\delta_{min}} \ \forall k,
\end{equation}
which is a contradiction. Therefore, there exists a large enough $\gamma_c$ such that $\mathrm{dist}(\bpi_{\gamma_c}^*(\state_t), \mathcal{U}^c(\state_t)) < \delta_{min}.$ By Lemma \ref{lemma:region}, a feasible input $\inp_t^c \in \mathcal{U}^c$ exists satisfying $\|\inp_t^c-\bpi^*_{\gamma_c}(\state_t)\| < \delta$, and thus $\Delta V(\state_t, \inp_t^c)< 0$, \qedhere

\end{proof}



\subsection{Multi-Step Feasible Descent}

Next, we show that the prior theorem generalizes when the policy is evaluated over a multi-step feedback policy $\bpi(\state_t, \bU) = (\bpi_t(\state_t, \bU), \cdots, \bpi_{t+\horizon-1}(\state_{t+\horizon-1}, \bU))$, where $\inp_t = \bpi_t(\state_t, \bU)$ and $\bU \in \Xi$. First, we define the multi-step feedback dynamics and cost.

\begin{definition}[Multi-Step Feedback Dynamics]
     The multi-step feedback dynamics is given as $p(\state_{t+\horizon} \mid \state, \bU) = \int \cdots \int \prod_{i=t}^{t+\horizon-1} p(\state_{i+1}|\state_{i},\bpi_i(\state_i, \bU)) d\state_{t+1} \cdots d\state_{t+\horizon-1}$.
    \label{def:multidynamics}
\end{definition}

\begin{definition}[Multi-Step Feedback Cost]
    We define
    \begin{align}
        \mathcal{G}_N(\state_t,\bU) &= \sum_{i=t}^{t+N-1}\beta^{i-t}\rlcostfunc_k(\state_{i},\bU) \\
        \rlcostfunc_k(\state_{i},\bU) &= \int g(\state_{i},\bpi_i(\state_i, \bU)) p(\state_{i}|\state_t,\bU)\, d\state_{i}
    \end{align}
    as the multi-step feedback cost, with analogous definitions for $\mathcal{L}_N(\state_t,\bU)$, $\costfunc(\state_{t},\bU)$, $\mathcal{C}_N(\state_t,\bU)$, and $\constraintfunc(\state_{t},\bU)$.
    \label{def:multicost}
\end{definition}

Given these, we show that there exists a feasible multi-step feedback policy within the neighborhood of an optimal multi-step feedback policy that descends the value function.

\begin{theorem}
    Under the assumptions of Theorem \ref{theorem:onestep} and a given $\state_t$, for a large enough $\gamma_c$, if $\Xi_\epsilon^c = \{\bU | \bpi_i(\state_i, \bU) \in \mathcal{U}^c(\state_i) \ \forall \ i \in \{t, \cdots, t+\horizon\}, V^\constraintfunc_{\bpi_t(\cdot,\bU)}(\state_t) < \epsilon\} \neq \varnothing$, where $\Xi$ is compact, $\bpi$ depends continuously on $\bU$, and the policy parameterization is rich enough to represent the optimal multi-step feedback policy, $\bpi(\cdot, \bU^*) = (\bpi_t^*, \cdots, \bpi_{t+\horizon-1}^*)$, then there exist feasible policy parameters $\bU_c$ in the neighborhood of $\bU^*$ that satisfies $\constraintfunc(\state_{i}, \bpi_i(\state_i, \bU_c))=0 \ \forall \ i \in \{t,...,t+\horizon-1\}$ and $\Delta V(\state_t, \bU_c) = \beta\mathbb{E}\left[V(\state_{t+\horizon}) \mid \state_t,\bU_c\right]-V(\state_t)<0$.
    \label{theorem:multistep}
\end{theorem}
\begin{proof}

    We show that the assumptions necessary for Theorem \ref{theorem:onestep} apply not only to the single-step dynamics and cost, but also to the multi-step versions.

    It follows from Definitions \ref{def:multidynamics} and \ref{def:multicost} and the assumptions above that the multi-step dynamics, cost, and constraint depend continuously on $\state_t$ and $\bU$. Likewise, the conditions that $\mathcal{L}_{\horizon}(\state_t,\bU) \succ 0$ and $\mathcal{C}_{\horizon}(\state_t,\bU) \succeq 0$ follow from definition \ref{def:multicost}.

    Finally, we observe that $V(\state_t)$ is not only the optimal value function for the single-step problem, but also for the multi-step feedback problem when the feedback policy is parametrized by $\bU \in \Xi$. This is shown through repeated application of Equation \ref{eq:sto_optimal_value_function} and the assumption that the policy parameterization is rich enough to represent the optimal multi-step feedback policy:
    \begin{align} 
        V(\state_t) &= \min_{\bU \in \Xi} \bigg[ \mathcal{G}_N(\state_t,\bU)+\beta\mathbb{E}\left[V(\state_{t+\horizon}) \mid \state_t, \bU \right] \bigg ].
    \end{align}
    
    By applying Theorem \ref{theorem:onestep} to $\bU$ instead of $\inp_t$, we show $\exists\bU_c$ s.t. $\Delta V(\state_t, \bU_c)<0$ and $\mathcal{C}_{\horizon}(\state_t,\bU_c) = 0$ which implies $\constraintfunc(\state_{i}, \bpi_i(\state_i, \bU_c))=0 \ \forall \ i \in \{t,...,t+\horizon-1\}$. \qedhere

\end{proof}

Thus, this analysis provides theoretical justification for using NMPC to search for feasible feedback policies in the neighborhood of the optimal policies which descend the value function while strictly adhering to constraints.

\subsection{Approximate Value Functions}

The preceding analysis assumes an optimal value function while our proposed algorithm uses learned actor and critic models. We acknowledge that in practice, these models converge to uncertain approximations of the optimal policy and value function. We define the expected change in the stochastic approximate value function as 
\begin{equation}
\Delta \hat{V}(\state_t,\inp_t) = \beta\mathbb{E}\left[\approxvaluefunc(\state_{t+1}) \mid \state_t,\inp_t\right]-\mathbb{E}[\approxvaluefunc(\state_t)].
\end{equation}

For a given $\state_t$ and $\inp_t$, there is a gap between the change in the approximate and optimal value functions
\begin{equation}
e_V(\state_t, \inp_t) = \Delta \hat{V}(\state_t,\inp_t) - \Delta V (\state_t,\inp_t).
\end{equation}
To guarantee that the optimal value function, $V(\state_t)$ actually decreases, the expected change in the approximate value function must decrease by more than this gap, 
\begin{equation}
\Delta \hat{V}(\state_t,\inp_t) < e_V(\state_t, \inp_t).
\end{equation}
Given that $V(\state_t)$ is unknown in practice, this gap is also unknown. Thus, we simply aim to ensure that $\Delta \hat{V}(\state_t,\inp_t) < 0$. 

\vspace{-2mm}

\subsection{Simulation Experiment}
We explore this analysis empirically using a stochastic Dubins Car dynamical system. The state is given as $\state_t = [r_x \ r_y \ \theta]^T$ where $r_x$, $r_y$ denote the position and $\theta$ is the orientation. The input is given as $\inp_t = [v \ \dot{\theta}]^T$ where $v$ is the speed and $\dot{\theta}$ is the angular velocity. The nominal continuous dynamics are given as $\nominaldynamicsfunc(\state_t, \inp_t) = [v\cos(\theta), v\sin(\theta), \dot{\theta}]^T$. The stochastic, discrete-time dynamics applies Euler integration with $\Delta t=0.1$ sec and Gaussian process noise with $\boldsymbol{\Sigma}_f = [0.01 \ 0.01 \ 0.01]^T$.

We constructed a simple environment in which the goal is surrounded by dense obstacles arranged on a hexagonal lattice (Fig. \ref{fig:toy-environment}). Initial states are sampled at positions $7.5\mathrm{m}$ away from the goal with the orientation pointing towards the goal.

We trained 13 actor-critic networks with increasingly large constraint penalties, $\gamma_c$. The inputs into these networks were the normalized position, sine/cosine of the orientation, and the normalized range and sine/cosine of the bearing to each obstacle. Each network was trained for 5 million timesteps.

Across 100 trials with differing initial states, we compared the actor policy against AC-PAC-NMPC for each $\gamma_c$. We evaluated the percentage of trials in which the system collided with obstacles and reached the goal without colliding. Additionally, we evaluated our RL-based warm-start with and without the actor infeasibility check (lines 6-8 in Algo. \ref{algo:warm-start}).

As shown in Figure \ref{fig:toy-results}, AC-PAC-NMPC was able to successfully navigate the environment at smaller $\gamma_c$ values than the actor policy alone. This supports Theorem \ref{theorem:multistep} empirically and shows that for a large enough penalty, a feasible policy exists in the neighborhood of the actor policy which descends the value function. Additionally, these results show that by warm-starting with the prior policy when the actor trajectory is infeasible, the safety of AC-PAC-NMPC can be significantly improved.

\begin{figure}[t]
    \centering
    \begin{minipage}[t]{0.33\columnwidth}
        \centering
        \includegraphics[trim={105 45 105 45},clip,width=\columnwidth]{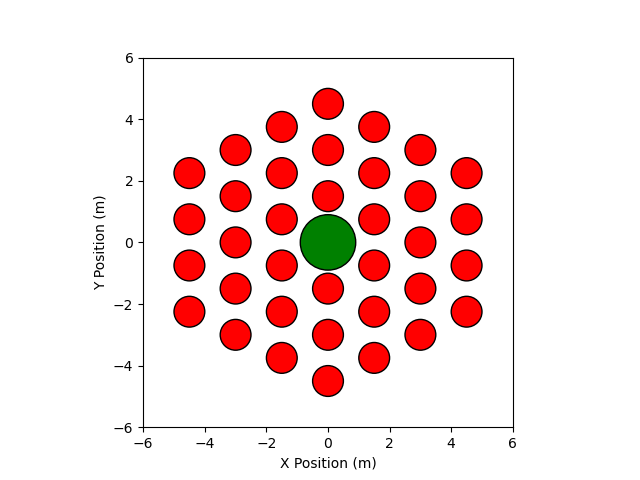}
        \caption{Green goal, red obstacles.}
        \label{fig:toy-environment}
    \end{minipage}\hfill
    \begin{minipage}[t]{0.63\columnwidth}
    \centering
    \includegraphics[trim={0 0 0 0},clip,width=\columnwidth]{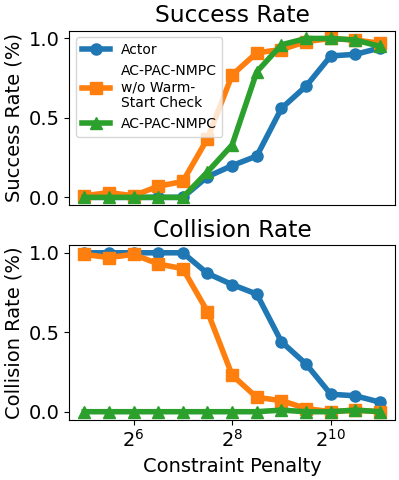}
    \caption{Success rate and collision rate for each constraint penalty value.}
    \label{fig:toy-results}
    \end{minipage}
\end{figure}

\section{Conclusion}

In this paper, we present an approach for RL-guided SNMPC to enable long-range navigation in unknown environments while enforcing constraints on probabilistic guarantees of safety. The proposed approach utilized learned actor-critic and sensor prediction models to incorporate RL-based warm starts, an uncertainty-aware terminal value function, and a value function improvement constraint into PAC-NMPC. We demonstrated that our approach, both in simulation and on hardware, is capable of navigating a robotic system with high-dimensional, nonlinear, underactuated, stochastic dynamics and local perception through an unknown environment in real-time. We showed that our approach was more robust to model mismatch, sim-to-real transfer, and distribution shift than the RL actor alone. Further, our approach achieved superior long-range performance than PAC-NMPC alone.

There are many interesting directions for future research. While it is beneficial that our approach can utilize actor-critic models trained with standard algorithms, it may be interesting to train it jointly with sensor prediction and PAC-NMPC in the loop. Although Monte Carlo dropout provides a fast approximation of uncertainty in the learned models, our approach would be improved by incorporating a method to calibrate, and potentially bound, the uncertainty estimates. Finally, further analysis into the gap between the estimated value function and the true value function could be explored.

\bibliographystyle{IEEEtran}
\bibliography{references.bib}

\end{document}